\documentclass{article}
\pdfoutput=1

\usepackage[preprint]{neurips_2019}

\usepackage[utf8]{inputenc} 
\usepackage[T1]{fontenc}    

\usepackage{hyperref}       
\usepackage{url}            

\usepackage{booktabs}       
\usepackage[para,flushleft]{threeparttable}
\usepackage{pifont}
\usepackage{amsfonts}       
\usepackage{nicefrac}       
\usepackage{microtype}      
\usepackage{amsmath,amsthm,amssymb}

\usepackage{comment}
\usepackage{float}
\usepackage{mathtools}
\usepackage{wrapfig}
\usepackage{tikz}
\usepackage{boxedminipage}
\usepackage{subcaption}

\usepackage{thmtools}
\usepackage{thm-restate}
\usepackage[capitalize]{cleveref}
\usepackage{wrapfig}

\usepackage{xcolor}
\definecolor{cardinal}{rgb}{0.64, 0.0, 0.0}
\usepackage{algorithm}
\usepackage{algpseudocode}

\usepackage[customcolors]{hf-tikz}
\usetikzlibrary{shapes,decorations,arrows,calc,arrows.meta,fit,positioning}
\tikzset{
    -Latex,auto,node distance =1.2 cm and 1.2 cm,semithick,font=\scriptsize,
    state/.style ={circle, draw, minimum width = 0.6 cm},
    point/.style = {circle, draw, inner sep=0.04cm,fill,node contents={}},
    bidirected/.style={Latex-Latex,dashed},
    el/.style = {inner sep=2pt, align=left, sloped}
}
\usepackage{xstring}
\newcommand{\dotnode}[4][]{
    \node[circle,draw, minimum width = 0.4 cm,label=#4] (#2) [#3] {$\ $};
    \node (b#2) [below = -2.7mm of #2] {\IfSubStr{#1}{lb}{\color{blue}{$\bullet$}}{\IfSubStr{#1}{rb}{\color{red}{$\bullet$}}{\IfSubStr{#1}{b}{$\bullet$}{}}}};
    \node (t#2) [above = -2.7mm of #2] {\IfSubStr{#1}{lt}{\color{blue}{$\bullet$}}{\IfSubStr{#1}{rt}{\color{red}{$\bullet$}}{\IfSubStr{#1}{t}{$\bullet$}{}}}};
}

\declaretheorem[name=Lemma,numberwithin=section]{lemma}

\declaretheorem[name=Theorem,numberwithin=section]{theorem}
\declaretheorem[name=Definition,numberwithin=section]{definition}

\newcommand{\cmark}{\ding{51}}%
\newcommand{\xmark}{\ding{55}}

\newcommand{\E}{{\rm I\kern-.3em E}}

\makeatletter
\newcommand*{\inlineequation}[2][]{%
  \begingroup
    \refstepcounter{equation}%
    \ifx\\#1\\%
    \else
      \label{#1}%
    \fi
    \relpenalty=10000 %
    \binoppenalty=10000 %
    \ensuremath{%
      #2%
    }%
    ~\@eqnnum
  \endgroup
}
\makeatother

\algdef{SE}[DOWHILE]{Do}{DoWhile}{\algorithmicdo}[1]{\algorithmicwhile\ #1}%

\title{Efficient Identification in Linear Structural\\Causal Models with Instrumental Cutsets}

\author{%
  Daniel Kumor\\
  Purdue University\\
  \texttt{dkumor@purdue.edu} \\
  \And
  Bryant Chen \\
  Brex Inc.\\
  \texttt{bryant@brex.com}\\
  \And 
  Elias Bareinboim \\
  Columbia University\\
  \texttt{eb@cs.columbia.edu}\\
}

\begin{document}

\maketitle

\begin{abstract}
One of the most common mistakes made when performing data analysis is attributing causal meaning to regression coefficients.  Formally, a causal effect can only be computed if it is identifiable from a combination of observational data and structural knowledge about the domain under investigation \cite[Ch.~5]{pearlCausalityModelsReasoning2000}. 
 Building on the literature of instrumental variables (IVs), a plethora of methods has been developed to identify causal effects in  linear systems. 
 Almost invariably, however, the most powerful such methods rely on exponential-time procedures. 
 In this paper, we investigate graphical conditions to allow efficient identification in arbitrary linear structural causal models (SCMs). In particular, we develop a method to efficiently
    find unconditioned \textit{instrumental subsets}, which are generalizations of IVs that can be used to tame the complexity of many canonical algorithms found in the literature. 
    Further, we prove that determining whether an effect can be identified with TSID \citep{weihsDeterminantalGeneralizationsInstrumental2017}, 
    a method more powerful than unconditioned instrumental sets and other efficient identification algorithms, is NP-Complete. 
    Finally, building on the idea of flow constraints, we introduce a new and efficient criterion called \textit{Instrumental Cutsets} (IC), which is able to solve for parameters missed by all other existing polynomial-time  algorithms.
\end{abstract}

\section{Introduction}

\begingroup
\setlength\intextsep{0pt}


Predicting the effects of interventions is one of the fundamental tasks in the empirical sciences. 
Controlled experimentation is considered the ``gold standard'' in which one physically intervenes in the system and learn about the corresponding effects. In practice, however, experimentation is not always possible due costs, ethical constraints,  or technical feasibility  --  e.g., a self-driving car should not need to crash to recognize that doing so has negative consequences. 
In such cases, the agent must uniquely determine the effect of an action using observational data and its structural knowledge of the environment. This leads to the problem of \emph{identification} \citep{pearlCausalityModelsReasoning2000,bar:pea16}.  

Structural knowledge is usually represented as a structural causal model (SCM)\footnote{Such models are also referred to as structural equation models, or SEM, in the literature.} \citep{pearlCausalityModelsReasoning2000},
which represents the set of observed and unobserved variables and their corresponding causal relations. 
We focus on the problem of generic identification in linear, acyclic SCMs. In such systems, the value of each observed variable is determined by a linear combination of the values of its direct causes along with a latent error term $\epsilon$. This leads to a system of equations $X = \Lambda^T X + \epsilon$, where $X$ is the vector of variables, $\Lambda^T$ is a lower triangular matrix whose $ij$th element $\lambda_{ij}$ -- called \textit{structural parameter} -- is 0 whenever $x_i$ is not a direct cause of $x_j$, and $\epsilon$ is a vector of latent variables. 

Methods for identification in linear SCMs generally assume that variables are normally distributed \citep{chenGraphicalToolsLinear2014},
meaning that the observational data can be summarized with a covariance matrix $\Sigma$. This covariance matrix can be linked to the underlying structural parameters through the system of equations \inlineequation[eqn:covariance]{\Sigma = XX^T = (I-\Lambda)^{-T}\Omega(I-\Lambda)^{-1}}, 
where $\Omega$'s elements, $\epsilon_{ij}$, represent $\sigma_{\epsilon_i\epsilon_j}$, and $I$ is the identity matrix \citep{foygelHalftrekCriterionGeneric2012}.
The task of causal effect identification in linear SCMs can, therefore, be seen as solving for the target structural parameter 
$\lambda_{ij}$ using \cref{eqn:covariance}. If the parameter can be expressed in terms of $\Sigma$ alone, then it is said to be \emph{generically identifiable}.

Such systems of polynomial equations can be approached directly through the application of Gr\"obner bases \citep{garcia-puenteIdentifyingCausalEffects2010}. In practice, 
however, these methods are doubly-exponential \citep{bardetComplexityGrObner2002} in the number of structural parameters,
and become computationally intractable very quickly, incapable of handling causal graphs with more than 4 or 5 nodes \citep{foygelHalftrekCriterionGeneric2012}.

Identification in linear SCMs has been a topic of great interest for nearly a century \citep{wrightCorrelationCausation1921}, 
including much of the early work in econometrics  \citep{wright1928tariff,fisher1966identification,bowdenInstrumentalVariables1984, bekkerIdentificationEquivalentModels1994}. The computational aspects of the problem, however, have only more recently received attention from computer scientists and statisticians \citep[Ch.~5]{pearlCausalityModelsReasoning2000}. 

\begin{wraptable}[27]{r}{0.4\textwidth}
    \begin{center}
        \begin{threeparttable}
        \begin{tabular}{l | c | c}
            \multicolumn{3}{c}{Identification Power \& Efficiency} \\
            \toprule
            Algorithm & Power & Eff.?\\
            \midrule
            IV\tnote{a} & low & \cmark\\
            cIV\tnote{b} & medium & \cmark\tnote{f}\\
            IS\tnote{c} & medium & \cmark\\
            scIS\tnote{g} & high & \textcolor{cardinal}{? $\rightarrow$ \xmark\tnote{*}}\\
            gIS\tnote{c} & high & ?\tnote{$\dagger$}\\
            HTC\tnote{e} & high & \cmark\\
            gHTC\tnote{h,j} & very high & \textcolor{cardinal}{? $\rightarrow$ \cmark\tnote{$\ddagger$}}\\
            cAV \& AVS\tnote{i} & very high & \textcolor{cardinal}{? $\rightarrow$ \cmark\tnote{$\ddagger$}}\\
            \textcolor{cardinal}{\textbf{Our Method}} & \textcolor{cardinal}{\textbf{very high}} & \textcolor{cardinal}{\cmark}\\
            gAVS\tnote{i} & very high & ?\tnote{$\dagger$}\\
            tsIV \& gHTC\tnote{j} & very high & \textcolor{cardinal}{? $\rightarrow$ \xmark\tnote{*}}\\
            Gr\"obner\tnote{d} & complete & \xmark\\
            \bottomrule
            \end{tabular}
            \begin{tablenotes}
                \tiny
                \textsuperscript{a}\citet{wright1928tariff};
                \textsuperscript{b}\citet{bowdenInstrumentalVariables1984};
                \textsuperscript{c}\citet{britoGeneralizedInstrumentalVariables2002};
                \textsuperscript{d}\citet{garcia-puenteIdentifyingCausalEffects2010};
                \textsuperscript{e}\citet{foygelHalftrekCriterionGeneric2012};
                \textsuperscript{f}\citet{vanderzanderEfficientlyFindingConditional2015};
                \textsuperscript{g}\citet{vanderzanderSearchingGeneralizedInstrumental2016};
                \textsuperscript{h}\citet{chenIdentificationOveridentificationLinear2016};
                \textsuperscript{i}\citet{chenIdentificationModelTesting2017};
                \textsuperscript{j}\citet{weihsDeterminantalGeneralizationsInstrumental2017}.
                \item[$\dagger$] Finding conditioning set for candidates shown to be NP-hard by (g), but complexity of search is open question.\\
                \item[$\ddagger$] Previous algorithms exponential without bounded input degree.\\
                \item[*] We proved NP-Completeness of this method.
            \end{tablenotes}
        \end{threeparttable}
    \end{center}
    \caption{Our contributions in relation to the literature are shown in \textcolor{cardinal}{\textbf{red}}; Ordered roughly by identification power. $?\rightarrow$ represents methods for which we determined complexity in this work.}
    \label{table:comparison}
\end{wraptable}

Since then, there has been a growing body of literature developing successively more sophisticated methods with increasingly stronger identification power i.e., capable of covering a larger spectrum of identifiable effects. 
Deciding whether a certain structural parameter can be identified in polynomial time is currently an open problem. 

The most popular identification method found 
 in the literature today is known as the \textit{instrumental variable} (IV) \citep{wright1928tariff}. A number of extensions of IVs have been proposed, including \emph{conditional IV} (cIV),  \citep{bowdenInstrumentalVariables1984,vanderzanderEfficientlyFindingConditional2015}, unconditioned \emph{instrumental sets} (IS) \citep{debritoGraphicalMethodsIdentification2004}, 
and the \emph{half-trek criterion} (HTC) \citep{foygelHalftrekCriterionGeneric2012}, all of which are accompanied with efficient, polynomial-time algorithms.

In contrast, generalized instrumental sets (gIS) \citep{britoGeneralizedInstrumentalVariables2002} were developed as a graphical criterion, 
without an efficient algorithm. 
\citet{vanderzanderSearchingGeneralizedInstrumental2016} proved that checking existence of conditioning sets that satisfy the gIS given a set of instruments is NP-Hard. They further proposed a simplified version of the criterion (scIS), for which finding a conditioning set can be done efficiently. It remains an open problem whether instruments satisfying the gIS can be found in polynomial time.

The generalized HTC (gHTC) \citep{chenIdentificationOveridentificationLinear2016,weihsDeterminantalGeneralizationsInstrumental2017}
and auxiliary variables (AVS) \citep{chenIncorporating2016,chenIdentificationModelTesting2017} were developed with algorithms that were polynomial, 
provided that the number of incoming edges to each node in the causal graph were bounded. The corresponding  algorithms are exponential without this restriction,
since they enumerate all subsets of each node's incident edges.

More recently, \citet{weihsDeterminantalGeneralizationsInstrumental2017} showed how constraints stemming from determinants of minors in the covariance matrix \citep{sullivantTrekSeparationGaussian2010}
can be exploited for identification (TSID). Still, the complexity of their method was left as an open problem. We use the term \textit{tsIV} to refer to the unnamed criterion underlying the TSID algorithm.
Against this background, our contributions are as follows: 

\begin{itemize}
    \item We develop an efficient algorithm for finding instrumental subsets, which overcomes the need for enumerating all subsets of incoming edges into each node. This leads to efficient identification algorithms exploiting the gHTC and AVS criteria.
    \item We prove NP-Completeness of finding tsIVs and scIS, which shows they are impractical for use in large graphs without constraining their search space.
    \item Finally, we introduce a new criterion, called \textit{Instrumental Cutsets}, and develop an associated polynomial-time identification algorithm. We show that ICs subsume both gHTC and AVS.
\end{itemize}

For the sake of clarity, a summary of our results in relation to existing literature is shown in  \Cref{table:comparison}.

\endgroup

\section{Preliminaries}

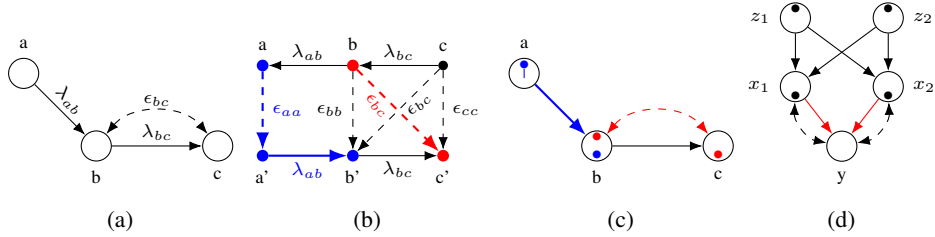
\begin{figure}
\center
    \begin{subfigure}[t]{0.22\linewidth}
        \center
		\begin{tikzpicture}
			\dotnode[]{1}{}{above:a};
        \dotnode[]{2}{below right=of 1,xshift=-1.5em,yshift=1.5em}{below:b};
        \dotnode[]{3}{right=of 2}{below:c};

        \path (1) edge node[el,above] {$\lambda_{ab}$} (2);
        \path (2) edge node[el,above] {$\lambda_{bc}$} (3);
        \path[bidirected] (2) edge[bend left=50] node[el,above] {$\epsilon_{bc}$}  (3);
        \end{tikzpicture}
        \caption{\label{fig:iv}}
    \end{subfigure}%
    \begin{subfigure}[t]{0.25\linewidth}
        \center
		\begin{tikzpicture}
			\node[red,very thick] (1) [label=above:b,point];
            \node (2) [right of = 1,point,label=above:c];
            \node[blue,very thick] (3) [ left of = 1,point,label=above:a];
            \node[blue,very thick] (4) [below of=1,point,label=below:b'];
            \node[red,very thick] (5) [right of = 4,point,label=below:c'];
            \node[blue,very thick] (6) [ left of = 4,point,label=below:a'];
			
            \path (2) edge node[el,above] {$\lambda_{bc}$} (1);
            \path (1) edge node[el,above] {$\lambda_{ab}$} (3);
            \path (4) edge node[el,below] {$\lambda_{bc}$} (5);
            \path[blue,thick] (6) edge node[el,below] {$\lambda_{ab}$} (4);

            \path[dashed] (1) edge node[left] {$\epsilon_{bb}$} (4);
            \path[dashed] (2) edge node[right] {$\epsilon_{cc}$} (5);
            \path[dashed,blue,thick] (3) edge node[right] {$\epsilon_{aa}$} (6);

            \path[dashed,red,thick] (1) edge node[el,below] {\hspace{-2em}$\epsilon_{bc}$} (5);
            \path[dashed] (2) edge node[el,below] {\hspace{2em}$\epsilon_{bc}$} (4);
        \end{tikzpicture}
        \caption{\label{fig:iv_flow}}
    \end{subfigure}
    \begin{subfigure}[t]{0.22\linewidth}
        \center
		\begin{tikzpicture}
			\dotnode[lt]{1}{}{above:a};
        \dotnode[rtlb]{2}{below right=of 1,xshift=-1.5em,yshift=1.5em}{below:b};
        \dotnode[rb]{3}{right=of 2}{below:c};

        \path[blue,thick] (1) edge (2);
        \path[blue,-] (t1) edge (b1);
        \path (2) edge (3);
        \path[bidirected,red] (2) edge[bend left=50] (3);
        \end{tikzpicture}
        \caption{\label{fig:iv_t2}}
    \end{subfigure}%
    \begin{subfigure}[t]{0.2\linewidth}
        \center
        \begin{tikzpicture}[node distance =.5 cm and .5 cm]
            \dotnode[]{y}{}{below:y};
            \dotnode[b]{x1}{above left=of y,xshift=0.5em}{left:$x_1$};   
            \dotnode[b]{x2}{above right=of y,xshift=-0.5em}{right:$x_2$};
            \dotnode[t]{z1}{above =of x1}{left:$z_1$};
            \dotnode[t]{z2}{above =of x2}{right:$z_2$};

            \path (z1) edge (x1);
            \path (z2) edge (x1);
            \path (z1) edge (x2);
            \path (z2) edge (x2);

            \path[red] (x1) edge (y);
            \path[red] (x2) edge (y);

            \path[bidirected] (x1) edge[bend right=40] (y);
            \path[bidirected] (x2) edge[bend left=40] (y);
            
        \end{tikzpicture}

        \caption{\label{fig:ivset}}
    \end{subfigure}%
    
    \caption{Conversion of an instrumental variable (a) into a trek-encoding flow graph shown in (b) if one ignores the colorings. From it, we can deduce that $\sigma_{bc}= \lambda_{ab}\epsilon_{aa}\lambda_{ab}\lambda_{bc} + \epsilon_{bb}\lambda_{bc}$. (c) shows another way of drawing the sets from (b), which facilitates interpretation in more complex settings. (d) $z_1,z_2$ can be used as an instrumental set to solve for $\lambda_{x_1y}$.}
    \label{fig:flowgraph}
\end{figure}

The causal graph of an SCM is defined as a triple $G=(V,D,B)$, where $V$ represents the nodes, $D$ the directed edges, and $B$ the bidirected ones. A linear SCM's graph has a node $v_i$ for each variable $x_i$, 
a directed edge between $v_i$ and $v_j$ for each non-zero $\lambda_{ij}$, and a bidirected edge between $v_i$ and $v_j$ for each non-zero $\epsilon_{ij}$ (\cref{fig:iv}). Each edge in the graph, therefore, corresponds
to an unknown coefficient, which we call a \emph{structural parameter}. When clear from the context, we will use $\lambda_{ij}$ and $\epsilon_{ij}$ to refer to the corresponding directed and bidirected edges in the graph. 
We define $Pa(x_i)$ as the set of parents of $x_i$, $An(x_i)$ as ancestors of $x_i$, $De(x_i)$ as descendants of $x_i$, 
and $Sib(x_i)$ as variables connected to $x_i$ with bidirected edges (i.e., variables with latent common causes).

We will refer to paths in the graph as ``unblocked" conditioned on a set $W$ (which may be empty), if they contain a collider ($a\rightarrow b\leftarrow c,a\leftrightarrow b \leftrightarrow c,a\rightarrow b \leftrightarrow c$) only when $b\in W\cup An(W)$, and if they do not otherwise contain vertices from $W$ (see d-separation, \cite{kollerProbabilisticGraphicalModels2009}). Unblocked paths without conditioning do not contain colliders, and are referred to as treks \citep{sullivantTrekSeparationGaussian2010}. The computable covariances of observable variables and the unknown structural parameters given in \cref{eqn:covariance} have a graphical interpretation in terms of a sum over all treks between nodes in the causal graph, namely $\sigma_{xy} = \sum \pi(x,y)$, where $\pi$ is the product of structural parameters along the trek.

Since unblocked paths in the causal graph have a non-trivial relation to arrow directions, we follow \cite{foygelHalftrekCriterionGeneric2012} in constructing an alternate DAG, called the ``flow graph" ($G_{flow}$), which encodes treks as directed paths between nodes (see \cref{fig:iv_flow}, where the blue path shows a trek between A and B in \cref{fig:iv}, meaning $\sigma_{ab}=\epsilon_{aa}\lambda_{ab}$). 
When referencing the flow graph, we call the ``top'' nodes (e.g., $a,b,c$ in \cref{fig:iv_flow}) the ``source nodes", and the ``bottom'' nodes ($a',b',c'$) the ``sink nodes".

Treks between two sets of nodes in $G$ are said to have ``no sided intersection" if they do not intersect in $G_{flow}$. The red and blue paths of Fig. \ref{fig:iv_flow} show such a set from $\{a,b\}$ to $\{b',c'\}$. Non-intersecting path sets are related to minors of the covariance matrix, denoted with $\Sigma_{(a,b),(b,c)}$\footnote{Refer to \cite{sullivantTrekSeparationGaussian2010} and the Gessel-Viennot-Lindstr\"om Lemma (\ref{lemma:gvl}) \citep{gesselDeterminantsPathsPlane1989}.}. We visually denote the source and sink sets in the original graph with a dot near the top of the node if it is a source, and a dot near the bottom if it is a sink. By these conventions, \cref{fig:iv_flow} can be represented by \cref{fig:iv_t2}.

For simplicity, we will demonstrate many of our contributions in the context of instrumental sets:
\begin{definition}
    \citep{britoGeneralizedInstrumentalVariables2002} A set $Z$ is called an \textbf{instrumental set (IS)} relative to $X\subseteq Pa(y)$ if
    (i) there exists an unblocked path set without sided intersection between $Z$ and $X$, and 
    (ii) there does not exist an unblocked path from any $z\in Z$ to $y$ in $G$ with edges $X\rightarrow y$ removed.
\end{definition}
In \cref{fig:ivset},  $\{z_1 , z_2\}$ is an instrumental set relative to $\{x_1,x_2\}$, leading to a system of equations 
solvable for $\lambda_{x_1y},\lambda_{x_2y}$. A conditioning set $W$ can be added to block paths from $Z$ to $y$, creating the \emph{simple conditional IS} (scIS).
If each $z_i$ has its own conditioning set, it is called a \emph{generalized IS} (gIS).

A set of identified structural parameters $\Lambda^*$ can be used to create ``auxiliary variables" \citep{chenIncorporating2016} which subtract out parents of variables whose effect is known: $x_i^* = x_i - \sum_{\lambda_{x_jx_i}\in \Lambda^*} \lambda_{x_jx_i} x_j$. Using these variables as instruments leads to AV sets (AVS), which are equivalent to the gHTC\footnote{Full definitions for the methods mentioned in this section are available in \cref{sec:prevdef}.}. 
Finally, we will build upon the tsIV, which exploits flow constraints in $G_{flow}$ to identify parameters:
\begin{definition}
\citep{weihsDeterminantalGeneralizationsInstrumental2017}
\label{def:tsiv}
Sets $S,T\subset V$, $|S|=|T|+1=k$ are a \textbf{tsIV} with respect to $\lambda_{xy}$ if (i) $De(y)\cap T=\emptyset$ (ii) The max flow between $S$ and $T'\cup\{x'\}$ in $G_{flow}$ is $k$, (iii) The max flow between $S$ and $T'\cup\{y'\}$ in $G_{flow}$ with $x'\rightarrow y'$,$w_i'\rightarrow y'$, $\lambda_{w_iy}\in \Lambda^*$ removed is less than $k$.
\end{definition}

\section{Efficiently Finding Instrumental Subsets}

\begin{figure}
   
    \begin{subfigure}[t]{0.35\linewidth}
        \center
        
    \begin{tikzpicture}[node distance =.6 cm and .6 cm]
        \dotnode[]{y}{}{below:y};
        \dotnode[]{x3}{above =of y}{right:$x_3$};
        \dotnode[b]{x2}{left=of x3}{left:$x_2$};
        \dotnode[b]{x1}{left=of x2}{below:$x_1$};   
        \dotnode[]{x4}{right=of x3}{right:$x_4$};
        \dotnode[b]{x5}{right=of x4}{below:$x_5$};
    
        \dotnode[]{z2}{above =of x3,xshift=-0.9cm}{above:$z_2$};
        \dotnode[t]{z1}{above =of x1}{above:$z_1$};
        \dotnode[t]{z3}{right =of z2}{above:$z_3$};
        \dotnode[t]{z4}{above =of x5}{above:$z_4$};
    
        \path[red] (x1) edge (y);
        \path[red] (x2) edge (y);
        \path (x3) edge (y);
        \path (x4) edge (y);
        \path[red] (x5) edge (y);
    
        \path[bidirected] (x1) edge[bend right=30] (y);
        \path[bidirected] (x2) edge[bend right=10] (y);
        \path[bidirected] (x3) edge[bend left=30] (y);
        \path[bidirected] (x4) edge[bend left=10] (y);
        \path[bidirected] (x5) edge[bend left=30] (y);
    
        \path (z1) edge (x1);
        \path (z1) edge (x5);
        \path (z2) edge (x2);
        \path (z2) edge (x5);
        \path (z2) edge (x4);
        \path (z2) edge (x3);
        \path (z3) edge (x2);
        \path (z3) edge (x5);
        \path (z4) edge (x1);
        \path (z4) edge (x2);
    \end{tikzpicture}
        \caption{\label{fig:ivsubset}}
    \end{subfigure}%
     \begin{subfigure}[t]{0.23\linewidth}
        \center
        \begin{tikzpicture}[node distance =0.5 cm and 0.5 cm]
            \dotnode[]{y}{}{below:y};
            \dotnode[b]{x2}{above=of y}{left:$x_2$};
            \dotnode[b]{x1}{left=of x2}{left:$x_1$};
            \dotnode[]{x3}{right=of x2}{above:$x_3$};
            \dotnode[]{w}{above=of x2}{left:$w$};

            \dotnode[t]{z1}{above left=of w}{left:$z_1$};
            \dotnode[t]{z2}{above right=of w}{left:$z_2$};

            \path (z1) edge (w);
            \path (z1) edge (x1);
            \path (z2) edge (w);
            \path (w) edge (x2);
            \path (w) edge (x3);
            \path[red] (x1) edge (y);
            \path (x2) edge (y);
            \path (x3) edge (y);
            \path[bidirected] (x1) edge[bend left=20] (y);
            \path[bidirected] (x2) edge[bend left=20] (y);
            \path[bidirected] (w) edge[bend left=25] (y);

            \path[bidirected] (x3) edge[bend left=20] (y);
        \end{tikzpicture}
        \caption{\label{fig:cutsetiv}}
    \end{subfigure}%
    \begin{subfigure}[t]{0.2\linewidth}
        \center
        \begin{tikzpicture}[node distance =.4 cm and .4 cm]
            \dotnode[]{y}{}{left:y};
            \dotnode[]{x}{above=of y}{left:x};
            \dotnode[]{b}{above=of x}{left:b};
            \dotnode[]{a}{above=of b}{left:a};
            \dotnode[]{c}{below=of y}{left:c};
            
            \path (a) edge (b);
            \path (b) edge (x);
            \path (y) edge (c);
            \path[red] (x) edge (y);
            \path[bidirected] (x) edge[bend left=40] (y);
            \path[bidirected] (a) edge[bend left=40] (b);
            \path[bidirected] (a) edge[bend left=50] (y);
            \path[bidirected] (a) edge[bend left=60] (c);
            
        \end{tikzpicture}

        \caption{\label{fig:nodiv}}
    \end{subfigure}
    \begin{subfigure}[t]{0.20\linewidth}
        \center
		\begin{tikzpicture}[node distance=0.6cm]
			\dotnode[]{y}{}{below:y};
			\dotnode[b]{x1}{above left=of y}{below:$x_1$};
			\dotnode[]{x2}{above right=of y}{above left:$x_2$};
			\dotnode[t]{z1}{above=of x1}{above:$z_1$};
			\dotnode[tb]{z2}{above=of x2}{above left:$z_2$};
			\dotnode[tb]{w}{above=of z2}{left:$w$};
			
			\path[red] (x1) edge (y);
			\path (x2) edge (y);
			\path (z1) edge (x1);
			\path (z2) edge (x2);
			\path (w) edge[bend right=10] (z1);
			\path (w) edge (z2);
			\path (z2) edge (z1);
			
			\path[bidirected] (x1) edge[bend left=20] (y);
			\path[bidirected] (x2) edge[bend right=20] (y);
			\path[bidirected] (w) edge[bend left=60] (y);
			\path[bidirected] (z2) edge[bend left=20] (z1);

        \end{tikzpicture}
        \caption{\label{fig:cAVonly}}
    \end{subfigure}
    
    \caption{(a): only a subset of edges can be identified with an instrumental set, 
    and finding the maximal subset in arbitrary graphs was exponential in previous algorithms. (b): $\lambda_{x_1y}$ could previously only be identified using tsIVs, which we show are NP-hard to find. 
    (c): $\lambda_{xy}$ cannot be identified using tsIVs, but is identified through iterative application of \cref{thm:unconditionedcutset} ($\lambda_{bx}$ using $a$, $\lambda_{yc}$ using $x^*$, and $\lambda_{xy}$ using $c^*$). (d): $\lambda_{x_1y}$ is identifiable with cAV but is not captured by IC.}
    \label{fig:idmodels}
\end{figure}
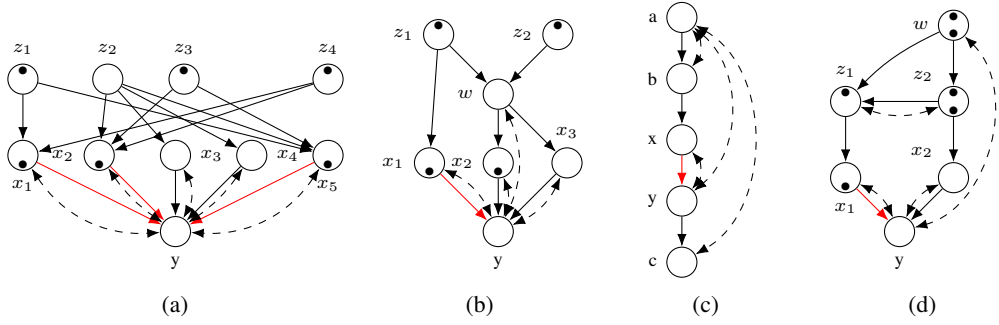
When identifying a causal effect, $\lambda_{x_1y}$, using instrumental sets, it is often the case that no instrument exists for $x_1$, but an instrumental set does exist for a subset of $y$'s parents that includes $x_1$. For example, in \cref{fig:ivset}, there does not exist any IV for $\{x_1\}$, but $\{z_1, z_2\}$ is an instrumental set for $\{x_1,x_2\}$, allowing identification of both $\lambda_{x_1y}$ and $\lambda_{x_2y}$.
Likewise, in \cref{fig:ivsubset}, $\{x_1,x_2,x_5\}$ is the \textit{only} subset of $Pa(y)$ which has a valid instrumental set ($\{z_1, z_3, z_4\}$).


One method for finding sets satisfying a criterion like IS would be to list all subsets of $y$'s incident edges, and for each subset, check if there exist corresponding variables $\{z_1,...,z_k\}$ satisfying all requirements. This is indeed the approach that algorithms developed for the gHTC \citep{chenIdentificationOveridentificationLinear2016,weihsDeterminantalGeneralizationsInstrumental2017} and AVS \citep{chenIncorporating2016} take. However, enumerating all subsets is clearly exponential in the number structural parameters / edges pointing to $y$. In this section, we show that finding this parameter subset can instead be performed in polynomial-time. 

First, we define the concept of ``match-blocking", which generalizes the above problem to arbitrary source and sink sets in a DAG, and can
be used to create algorithms for finding valid subsets applicable to IS, the gHTC, AVS, and our own identification criterion, instrumental cutsets (IC). 

\begin{restatable}{definition}{matchblock}
    \label{def:matchblock}
    Given a directed acyclic graph $G=(V,D)$, a set of source nodes $S$ and sink nodes $T$, the sets $S_f\subseteq S$ and $T_f\subseteq T$,  with $|S_f|=|T_f|=k$, are called \textbf{match-blocked} iff 
    for each $s_i\in S_f$, all elements of $T$ reachable from $s_i$ are in the set $T_f$, and the max flow between $S_f$ and $T_f$ is $k$ in $G$ where each vertex has capacity 1.
\end{restatable}



To efficiently find a match-block\footnote{While there exist methods for finding solvable subsystems of equations \citep{duffImplementationTarjanAlgorithm1978,sridharAlgorithmsStructuralDiagnosis1996,goncalvesNoteComplexityCausal2016}, 
they cannot be applied to our situation due to the requirement of nonintersecting paths in a full arbitrary DAG.}, we observe that if a max flow is done from a set of variables $S$ to $T$, then any element of $T$ that has 0 flow through it cannot be part of a match-block,
and therefore none of its ancestors in $S$ can be part of the match-block either:

\begin{restatable}{theorem}{subsets}
    \label{thm:matchblock}
    Given a directed acyclic graph $G=(V,D)$, a set of source nodes $S$, sink nodes $T$, and a max flow $\mathcal{F}$ from $S$ to $T$ in $G$ with vertex capacity 1, if a node $t_i\in T$ has 0 flow crossing it in $\mathcal{F}$, then there do not exist subsets $S_m\subseteq S,T_m\subseteq T$ where $S_m,T_m$ are match-blocked and $t_i \in T_m$. Furthermore, for any match-block $(S_m,T_m)$, we have $|S_m\cap An(t_i)|=0$.
\end{restatable}

This suggests an algorithm for finding the match-block: find a max flow from $S$ to $T$, then remove elements of $T$ that did not have
a flow through them, and all of their ancestors from $S$, and repeat until no new elements are removed.
The procedure, given in \cref{alg:matchblock}, runs in polynomial time. 

%
%

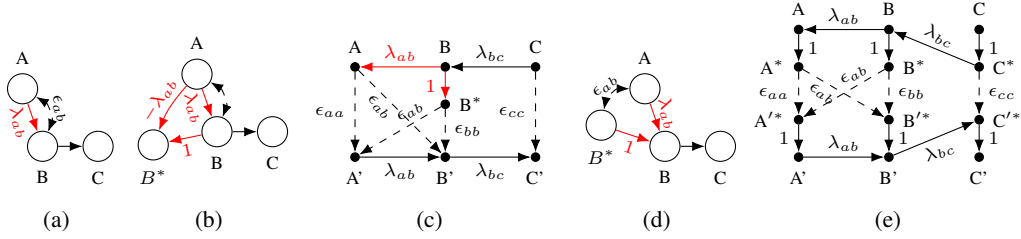
\begin{figure}
    \begin{subfigure}[t]{0.13\linewidth}
        \center
		\begin{tikzpicture}[node distance =1 cm and 0.5 cm]
			\dotnode[]{1}{}{above:A};
        \dotnode[]{2}{below right=of 1,xshift=-1.5em,yshift=1.5em}{below:B};
        \dotnode[]{3}{right=of 2,xshift=-0.5em}{below:C};

        \path[red] (1) edge node[el,below] {$\lambda_{ab}$} (2);
        \path (2) edge (3);
        \path[bidirected] (1) edge[bend left=40] node[el,above] {$\epsilon_{ab}$} (2);
        \end{tikzpicture}
        \caption{\label{fig:auxexample}}
    \end{subfigure}%
    \begin{subfigure}[t]{0.16\linewidth}
        \center
		\begin{tikzpicture}[node distance =1 cm and 0.5 cm]
			\dotnode[]{1}{}{above:A};
        \dotnode[]{2}{below right=of 1,xshift=-1.5em,yshift=1.5em}{below:B};
        \dotnode[]{3}{right=of 2,xshift=-0.5em}{below:C};
        \dotnode[]{4}{below left=of 1,yshift=1em,xshift=0.6em}{below:$B^*$};

        \path[red] (1) edge node[el,below] {$\lambda_{ab}$} (2);
        \path (2) edge (3);

        \path[red] (2) edge node[el,below] {$1$} (4);
        \path[red] (1) edge[bend right=10] node[el,above] {$-\lambda_{ab}$} (4);

        \path[bidirected] (1) edge[bend left=40] (2);
        \end{tikzpicture}
        \caption{\label{fig:auxexamplechen}}
    \end{subfigure}%
    \begin{subfigure}[t]{0.26\linewidth}
        \center
		\begin{tikzpicture}
			\node (1) [label=above:B,point];
            \node (2) [right of = 1,point,label=above:C];
            \node (3) [ left of = 1,point,label=above:A];
            \node (4) [below of=1,point,label=below:B'];
            \node (5) [right of = 4,point,label=below:C'];
            \node(6) [ left of = 4,point,label=below:A'];

            \path[red] (1) edge node[el,above] {$\lambda_{ab}$} (3);
            
            \path (6) edge node[el,below] {$\lambda_{ab}$} (4);
            \path (4) edge node[el,below] {$\lambda_{bc}$} (5);

            \node(bs) [below of=1,yshift=2em,point,label=right:B$^*$];
            \path[dashed] (bs) edge node[right] {$\epsilon_{bb}$} (4);
            \path[dashed] (bs) edge node[el,above] {\hspace{2em}$\epsilon_{ab}$} (6);
            \path[red] (1) edge node[left] {$1$} (bs);
            \path (2) edge node[el,above] {$\lambda_{bc}$} (1);

            \path[dashed] (2) edge node[left] {$\epsilon_{cc}$} (5);
            \path[dashed] (3) edge node[left] {$\epsilon_{aa}$} (6);

            \path[dashed] (3) edge node[el,below] {\hspace{-2em}$\epsilon_{ab}$} (4);
        \end{tikzpicture}
        \caption{\label{fig:auxflow}}
    \end{subfigure}
    \begin{subfigure}[t]{0.16\linewidth}
        \center
		\begin{tikzpicture}[node distance =1 cm and 0.5 cm]
			\dotnode[]{1}{}{above:A};
        \dotnode[]{2}{below right=of 1,xshift=-1.5em,yshift=1.5em}{below:B};
        \dotnode[]{3}{right=of 2,xshift=-.5em}{below:C};
        \dotnode[]{4}{below left=of 1,yshift=2.3em,xshift=0.6em}{below:$B^*$};

        \path[red] (1) edge node[el,above] {$\lambda_{ab}$} (2);
        \path (2) edge (3);
        \path[bidirected] (1) edge[bend right=40] node[el,above] {$\epsilon_{ab}$} (4);

        \path[red] (4) edge node[el,below] {$1$} (2);
        \end{tikzpicture}
        \caption{\label{fig:auxexampleproposed}}
    \end{subfigure}%
    \begin{subfigure}[t]{0.28\linewidth}
        \center
		\begin{tikzpicture}
			\node (1) [label=above:B,point];
            \node (2) [right of = 1,point,label=above:C];
            \node (3) [ left of = 1,point,label=above:A];
            \node (bps) [below of=1,point,label=right:B$'^*$];
            \node (cps) [right of = bps,point,label=right:C$'^*$];
            \node(aps) [ left of = bps,point,label=left:A$'^*$];

            \path (1) edge node[el,above] {$\lambda_{ab}$} (3);
            

            \node(bs) [below of=1,yshift=2em,point,label=right:B$^*$];
            \node(cs) [below of=2,yshift=2em,point,label=right:C$^*$];
            \node(as) [below of=3,yshift=2em,point,label=left:A$^*$];

            \node(bp) [below of=bps,yshift=2em,point,label=below:B'];
            \node(cp) [below of=cps,yshift=2em,point,label=below:C'];
            \node(ap) [below of=aps,yshift=2em,point,label=below:A'];
            
            \path (1) edge node[left] {$1$} (bs);
            \path (2) edge node[right] {$1$} (cs);
            \path (3) edge node[right] {$1$} (as);
            \path (aps) edge node[left] {$1$} (ap);
            \path (bps) edge node[right] {$1$} (bp);
            \path (cps) edge node[right] {$1$} (cp);

            \path[dashed] (bs) edge node[right] {$\epsilon_{bb}$} (bps);
            \path[dashed] (bs) edge node[el,above] {\hspace{2em}$\epsilon_{ab}$} (aps);
            
            \path (cs) edge node[el,above] {$\lambda_{bc}$} (1);
            \path (bp) edge node[el,below] {$\lambda_{bc}$} (cps);

            \path[dashed] (cs) edge node[right] {$\epsilon_{cc}$} (cps);
            \path[dashed] (as) edge node[left] {$\epsilon_{aa}$} (aps);
            \path (ap) edge node[el,above] {$\lambda_{ab}$} (bp);

            \path[dashed] (as) edge node[el,below] {\hspace{-2em}$\epsilon_{ab}$} (bps);
        \end{tikzpicture}
        \caption{\label{fig:auxflow2}}
    \end{subfigure}
    
    \caption{The graph in (a) with edge $\lambda_{ab}$ known has an auxiliary variable $B^*$ shown in (b). (c) shows a modified flow graph, which encodes the treks from $B^*$ excluding the known edge.
    This corresponds to a new encoding of the AV, shown in (d). Finally, (e) gives the auxiliary flow graph as described in \cref{def:auxflowgraph}}
    \label{fig:auxgraph}
\end{figure}

%
%

\begin{algorithm}
    \caption{Find Maximal Match-Block given DAG $G$, source nodes $S$ and target nodes $T$}
    \label{alg:matchblock}
    \begin{algorithmic}
        \Function{MaxMatchBlock}{G,S,T}
        \Do
        \State $\mathcal{F} \leftarrow \textsc{MaxFlow}(G,S,T)$
        \State $T' \leftarrow \{ t_i | \mathcal{F}\text{ has 0 flow through $t_i\in T$}\}$
        \State $T \leftarrow T\setminus T'$
        \State $S \leftarrow S \setminus An(T')$
        \DoWhile {at least one element of $T$ was removed in this iteration}
        \State\Return $(S,T)$
        \EndFunction
        \end{algorithmic}
\end{algorithm}
The match-block can be exploited to find instrumental subsets by using $G_{flow}$ with ancestors of $y$ that don't have back-paths to siblings of $y$ as $S$ and $Pa(y)'$ as $T$. 
This procedure is shown to find an IS if one exists in \cref{cor:ivsubsets}, and is implemented in \cref{alg:IS} of the appendix.
%

\subsection{Extending IS to AVS with the Auxiliary Flow Graph}

A match-block operates upon a directed graph. When using instrumental sets, one can convert the SCM to the flow graph $G_{flow}$, but the AVS algorithm exploits auxiliary variables, 
which are not encoded in this graph. The covariance of an auxiliary variable with another variable $y$ can be written:
$$
\sigma_{b^*y} = \E[b^*y] = \E\left[\left(b - \sum_{\lambda_{x_ib}\in \Lambda^*} \lambda_{x_ib} x_i\right)y\right] = \sigma_{by} - \sum_{\lambda_{x_ib}\in \Lambda^*} \lambda_{x_ib} \sigma_{x_iy}
$$

This quantity behaves like $\sigma_{by}$ with treks from $b$ starting with the known $\lambda_{x_ib}$ removed. We can therefore construct a flow graph which encodes this intuition explicitly using \cref{def:auxflowgraph}. 
An example of an auxiliary flow graph can be seen in \cref{fig:auxgraph}, which uses $b^*=b-\lambda_{ab}a$. In \cref{fig:auxflow}, $B^*$ no longer has the edge $\lambda_{ab}$ to $A$, but it still has all other edges, giving it all of the same treks
as $B$, except the ones subtracted out in the AV. The original variable, $B$, has an edge with weight $1$ to $B^*$, making its treks identical to the standard flow graph.
With this new graph, and \textsc{MaxMatchBlock}, the algorithm for instrumental sets can easily be extended to find AVS (\cref{alg:AVS} in appendix), which in turn is equivalent to the gHTC. We have therefore shown that both of these methods can be efficiently applied for identification, without a restriction on number of edges incoming into a node.
\begin{restatable}{definition}{auxflowgraph} \textbf{(Auxiliary Flow Graph)}
    \label{def:auxflowgraph}
Given a linear SCM $(\Lambda,\Omega)$ with causal graph $G=(V,D,B)$, and set of known structural parameters $\Lambda^*$, the auxiliary flow graph is a weighted DAG with vertices $V\cup V^*\cup V'\cup V'^*$ and edges
\begin{enumerate}
\item $j\rightarrow i$ and $i'\rightarrow j'$ both with weight $\lambda_{ij}$ if $(i\rightarrow j)\in D$, and $\lambda_{ij}\in \Lambda^*$
\item $j^*\rightarrow i$ and $i'\rightarrow j'^*$ both with weight $\lambda_{ij}$ if $(i\rightarrow j)\in D$, and $\lambda_{ij}\notin \Lambda^*$
\item $i\rightarrow i^*$ and $i'^*\rightarrow i'$ with weight 1, and $i^*\rightarrow i'$ with weight $\epsilon_{ii}$ for $i\in V$
\item $i^*\rightarrow j'^*$ with weight $\epsilon_{ij}$ if $(i\leftrightarrow j)\in B$
\end{enumerate}
This graph is referred to as $G_{aux}$. The nodes without $'$ are called ``source'', or ``top'' nodes, and the nodes with $'$ are called ``sink'' or ``bottom'' nodes. The nodes with $*$ are called ``auxiliary" nodes.
\end{restatable}

\section{On Searching for tsIVs \citep{weihsDeterminantalGeneralizationsInstrumental2017}}
\label{sec:nphard}

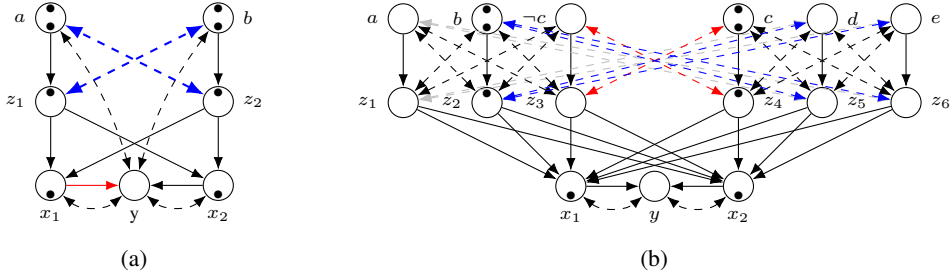
\begin{figure}

    \begin{subfigure}[t]{0.35\linewidth}
        \center
        \begin{tikzpicture}[node distance =.7 cm and .7 cm]
            \dotnode[]{y}{}{below:y};
            \dotnode[b]{x1}{left=of y}{below:$x_1$};
            \dotnode[b]{x2}{right=of y}{below:$x_2$};   
            \dotnode[t]{z1}{above=of x1}{left:$z_1$};
            \dotnode[tb]{a}{above=of z1}{left:$a$};
            \dotnode[t]{z2}{above=of x2}{right:$z_2$};
            \dotnode[tb]{b}{above=of z2}{right:$b$};

            \path[red] (x1) edge (y);
            \path (x2) edge (y);
            \path (z1) edge (x1);
            \path (z1) edge (x2);
            \path (a) edge (z1);
            \path (z2) edge (x1);
            \path (z2) edge (x2);
            \path (b) edge (z2);

            \path[bidirected] (x1) edge[bend right=40] (y);
            \path[bidirected] (x2) edge[bend left=40] (y);
            \path[bidirected] (a) edge[bend left=10] (y);
            \path[bidirected] (b) edge[bend right=10] (y);

            \path[bidirected,thick,blue] (z1) edge (b);
            \path[bidirected,thick,blue] (z2) edge (a);

        \end{tikzpicture}

        \caption{\label{fig:basediv_clause}}
    \end{subfigure}%
    \begin{subfigure}[t]{0.64\linewidth}
        \center
        \begin{tikzpicture}[node distance =.7 cm and .7 cm]
            \dotnode[]{y}{}{below:$y$};
            \dotnode[b]{x2}{right=of y}{below:$x_2$};   
            \dotnode[t]{z4}{above=of x2}{right:$z_4$};
            \dotnode[tb]{w4}{above=of z4}{right:$c$};
            \dotnode[]{z5}{right=of z4}{right:$z_5$};
            \dotnode[]{w5}{above=of z5}{right:$d$};
            \dotnode[]{z6}{right=of z5}{right:$z_6$};
            \dotnode[]{w6}{above=of z6}{right:$e$};

            \dotnode[b]{x1}{left=of y}{below:$x_1$};   
            \dotnode[]{z3}{above=of x1}{left:$z_3$};
            \dotnode[t]{z2}{left=of z3}{left:$z_2$};
            \dotnode[]{z1}{left=of z2}{left:$z_1$};
            \dotnode[]{w1}{above=of z1}{left:$a$};
            
            \dotnode[tb]{w2}{above=of z2}{left:$b$};
            
            \dotnode[]{w3}{above=of z3}{left:$\lnot c$};

            \path (x2) edge (y);
            \path (z1) edge (x2);
            \path (z1) edge (x1);
            \path (w1) edge (z1);
            \path (z2) edge (x2);
            \path (z2) edge (x1);
            \path (w2) edge (z2);
            \path (z3) edge (x2);
            \path (z3) edge (x1);
            \path (w3) edge (z3);

            \path[bidirected] (x2) edge[bend left=40] (y);
            \path[bidirected] (x1) edge[bend right=40] (y);

            \path (x1) edge (y);
            \path (z4) edge (x2);
            \path (z4) edge (x1);
            \path (w4) edge (z4);
            \path (z5) edge (x2);
            \path (z5) edge (x1);
            \path (w5) edge (z5);
            \path (z6) edge (x2);
            \path (z6) edge (x1);
            \path (w6) edge (z6);

            \path[bidirected,lightgray] (z1) edge (w5);
            \path[bidirected,lightgray] (w1) edge (z5);
            \path[bidirected,lightgray] (z1) edge (w6);
            \path[bidirected,lightgray] (w1) edge (z6);

            \path[bidirected] (z1) edge (w2);
            \path[bidirected] (z1) edge (w3);

            \path[bidirected] (z2) edge (w1);
            \path[bidirected] (z2) edge (w3);

            \path[bidirected] (z3) edge (w2);
            \path[bidirected] (z3) edge (w1);

            \path[bidirected] (z4) edge (w5);
            \path[bidirected] (z4) edge (w6);

            \path[bidirected] (z5) edge (w4);
            \path[bidirected] (z5) edge (w6);

            \path[bidirected] (z6) edge (w5);
            \path[bidirected] (z6) edge (w4);

            \path[bidirected,red] (z3) edge (w4);
            \path[bidirected,red] (z4) edge (w3);

            \path[bidirected,blue] (z2) edge (w5);
            \path[bidirected,blue] (w2) edge (z5);
            \path[bidirected,blue] (z2) edge (w6);
            \path[bidirected,blue] (w2) edge (z6);

        \end{tikzpicture}

        \caption{\label{fig:fulldivsat}}
    \end{subfigure}%

    \caption{The model in (a) encodes the basic structure we exploit in our proof. The set $\{z_1,z_2\}$ is a simple conditional instrumental set only if $z_1\leftrightarrow b$ and $z_2\leftrightarrow a$ do not exist,
    since it opens the collider from $z_1$ through $b$ to $y$.
    (b) shows the full encoding of a 1-in-3SAT for $(a\vee b \vee \lnot c)\wedge(c\vee d\vee e)$. \textbf{We have removed bidirected edges from all top nodes to $y$ in (b) for clarity}.}
    \label{fig:givsat}
\end{figure}

There exist structural parameters that can be identified using tsIVs, which cannot be found with gHTC, AVS, nor any other efficient method. However, there currently does not exist an efficient algorithm for finding tsIVs in arbitrary DAGs.
This section can be summarized with \cref{cor:divnp}:

\begin{restatable}{corollary}{divnp}
    \label{cor:divnp}
Given an SCM and target structural parameter $\lambda_{xy}$, determining whether there exists a tsIV which can be used to solve for $\lambda_{xy}$ in $G$ is an NP-Complete problem.
\end{restatable}

We show this by encoding 1-in-3SAT, which is NP-Complete \citep{schaeferComplexitySatisfiabilityProblems1978}, into a graph, 
such that finding a tsIV is equivalent to solving for a satisfying boolean assignment.

Since tsIVs can be difficult to intuitively visualize, we will illustrate the ideas behind the proof with simple conditional instrumental sets (scIS), which we also show are NP-Complete to find (\cref{cor:scivnp}).
We observe a property of the graph in \cref{fig:basediv_clause}: with $a\leftrightarrow z_2$ and $z_1 \leftrightarrow b$ removed (blue), $\{z_1,z_2\}$ can be used as scIS for $\lambda_{x_1y},\lambda_{x_2y}$, since their back-paths to $y$ ($z_1\leftarrow a \leftrightarrow y$ and $z_2\leftarrow b \leftrightarrow y$) can be blocked by conditioning on $W=\{a,b\}$.
However, if the bidirected edges are not removed, conditioning on $a$ or $b$ opens a path to $y$ using them as a collider ($z_1\leftrightarrow b\leftrightarrow y$ and $z_2 \leftrightarrow a \leftrightarrow y$). 
A simple conditional instrumental set for $\lambda_{x_1y}$ exists in this graph if and only if none of the instruments has bidirected edges to another instrument's required conditioning variable.

We exploit this property to construct the graph in \cref{fig:fulldivsat}, which repeats the structure from \cref{fig:basediv_clause} for each literal $l_i$ in the 1-in-3SAT formula $(a\vee b \vee \lnot c)\wedge(c\vee d\vee e)$. 
Each clause is designed so that usage of any potential instrument, $z_i$, precludes usage of the other two potential instruments in the clause. For example, the bidirected edges in \cref{fig:fulldivsat} from $z_3$ to $a$ and $b$ disallow usage of $z_1$ and $z_2$ as instruments once $z_3$ is used.
Likewise, there are \textbf{\color{red}bidirected edges} between the $c$ and $\lnot c$ structures, since if $c$ is $true$, $\lnot c$ cannot be $true$. Similarly, $b$ has \textbf{\color{blue}bidirected edges} to $z_5$ and $z_6$, since if $a$ is $true$, then $\lnot c$ is false and $c$ is true, so $d$ and $e$ must be disabled.
Finally, $y$ has 2 parents, $x_1,x_2$, corresponding to the two clauses. Each element of $Z$ has edges to all parents of $y$, meaning that any scIS existing in the graph must have as many instruments as there are clauses. 
Thus, finding an scIS for $y$ in this graph corresponds to finding a satisfying assignment for the formula. 
The full procedure for generating the graph, which is the same for both scIS and tsIV, is given in \cref{thm:divsat}.

\begin{restatable}{theorem}{divsat}
    \label{thm:divsat}
    Given a boolean formula $F$ in conjunctive normal form, if a graph $G$ is constructed as follows, starting from a target node $y$:
    \begin{enumerate}
    \item For each clause $c_i \in F$, a node $x_i$ is added with edges $x_i\rightarrow y$ and $x_i\leftrightarrow y$
    \item For $c_i \in F$, take each literal $l_j\in c_i$, and add nodes $z_{ij},w_{ij}$, 
    with edges $w_{ij} \rightarrow z_{ij}$, $w_{ij}\leftrightarrow y$, and $z_{ij}\rightarrow x_k$ $\forall x_k$
    \item For $c_i\in F$, $l_j,l_k\in c_i$ where $j\ne k$ add bidirected edge $z_{ij}\leftrightarrow w_{ik}$
    \item $\forall c_i,c_m \in F, l_j \in c_i, l_n\in c_m$ with $i\ne m$, add a bidirected edge $z_{ij} \leftrightarrow w_{mn}$ if 
        \begin{enumerate}
            \item $l_j = \lnot l_n$, or
            \item $\exists l_q \in c_m$ with $q\ne n$ and $l_j=l_q$, or
            \item $\exists l_p \in c_i, l_q\in c_m$ with $p\ne j$ and $q\ne n$ where $l_p=\lnot l_q$
        \end{enumerate}
    \end{enumerate}
    Then a tsIV exists for  $\lambda_{x_1y}$ in $G$ if and only if there is a truth assignment to the variables of $F$ such that there is exactly one true literal in each clause of $F$.
\end{restatable}

\section{The Instrumental Cutset Identification Criterion}
\label{sec:criterion}

While finding a tsIV is NP-hard, it is possible to create a new criterion, which both includes constraints from the tsIV that can be efficiently found, 
and exploits knowledge of previously identified edges similarly to AVS. This criterion is described in the following theorem.

\begin{restatable}{theorem}{unconditionedcutset}
    \label{thm:unconditionedcutset}
    \textbf{(Instrumental Cutset)} Let $M=(\Lambda,\Omega)$ be a linear SCM with associated causal graph $G=(V,D,E)$, a set of identified structural parameters $\Lambda^*$, and a target structural parameter $\lambda_{xy}$.
    Define $G_{aux}$ as the auxiliary flow graph for $(G,\Lambda^*)$.
    Suppose that there exist subsets $S\subset V\cup V^*$, with $V^*$ representing the set of AVs, and $T\subseteq Pa(y^*)\setminus \{x\}$ with $|S|=|T|-1=k$ such that
    \begin{enumerate}
    
\item There exists a flow of size $k$ in $G_{aux}$ from $S$ to $T\cup\{x\}$
\item There does not exist a flow of size $k$ from $S$ to $T\cup\{y\}$ in $G_{aux}$ with $x'\rightarrow y'^*$ removed
\item No element of $\{y\}\cup Sib(y)$ has a directed path to $s_i \in S$ in $G$

    \end{enumerate}
    then $\lambda_{xy}$ is generically identifiable by the equation:
    $$
    \lambda_{xy} = \frac{\det \Sigma_{S, T\cup\{y^*\}}}{\det \Sigma_{S, T\cup\{x\}}}
    $$
\end{restatable}

Instrumental Cutsets (ICs) differ from tsIVs in two fundamental ways:
\begin{enumerate}
\item We allow auxiliary variables, enabling exploitation of previously identified structural parameters incoming to $s_i \in S$ for identification.
An example that is identifiable with ICID, but not with TSID is shown in \cref{fig:nodiv}.
\item We require that $T$ is a subset of the parents of target node $y$, and that y has no half-treks to $S$. A version of IC which avoids these requirements is given in \cref{thm:auxiliarydiv} in the appendix.
While this version is strictly more powerful than tsIV, finding satisfying sets can be shown to be NP-hard by a modified version of the arguments given in \cref{sec:nphard} (\cref{sec:cacnphard}).
\end{enumerate}

$\lambda_{x_1 y}$ in Fig. \ref{fig:cutsetiv} is an example of a parameter that can be identified using ICs. To see this, consider the paths from $\{z_1,z_2\}$ to $y$ that do not have sided intersection anywhere but at $y$. 
One such path set is $z_1\rightarrow x_1\rightarrow y$ and $z_2\rightarrow w \rightarrow x_2\rightarrow y$. 
After removing the edge $x_1\rightarrow y$, there no longer exist 2 separate nonintersecting paths to $y$, because the node $w$ forms a bottleneck, or cut, allowing only one path to pass to $y$. According to theorem \ref{thm:unconditionedcutset}, this is sufficient to uniquely solve for $\lambda_{x_1y}$.

In contrast, previously known efficient algorithms cannot identify $\lambda_{x_1y}$. $z_1$, which is the only possible instrument for $x_1$, has unblockable paths to $y$ through $x_2$ and $x_3$ ($w$ cannot be conditioned, since it has a bidirected edge to $y$).
Furthermore, only $z_2$ is a possible additional instrument for $x_2$ or $x_3$, giving 2 candidate instruments $\{z_1,z_2\}$ for a set of 3 parents of $y$, $\{x_1,x_2,x_3\}$, 
all of which need to be matched to an instrument to enable solving for $\lambda_{x_1y}$.

In general, any coefficient that can be identified using the gHTC or AVS can also be identified using ICs. ICs, therefore, strictly subsume gHTC and AVS.
\begin{restatable}{lemma}{subsumesAVS}
If a structural parameter $\lambda_{xy}$ of linear SCM $M$ is identifiable using the gHTC or AVS then $\lambda_{xy}$ is identified using IC. 
There also exists a model $M'$ such that $\lambda_{xy}$ is identifiable using IC, but cannot be identified using gHTC or AVS.
\end{restatable}

Lastly, we discuss the identification power of ICs with respect to cAVs \citep{chenIdentificationModelTesting2017}, which are single auxiliary conditional instruments, and can be found in polynomial-time. While there are many examples of parameters that ICs, and even the gHTC and AVS, can identify that cAVs cannot, it turns out there are also examples that the cAV can identify that ICs cannot.  
This is because ICs, which operate on sets of variables, do not include conditioning. In \cref{fig:cAVonly}, $z_1$ is a cAV for $\lambda_{x_1y}$ when conditioned on $\{w,z_2\}$, but no IC exists, because $\lambda_{x_2x_1}$ cannot be identified, and therefore the back path from $x_1$ through $w\leftrightarrow y$ cannot be eliminated. While a version of ICs with conditioning could be developed, the algorithmic complexity of
finding parameters identifiable by such a criterion is unclear. 
A version of ICs with a single conditioning set would be NP-hard to find, which can be shown using a modified version of our results from \cref{sec:nphard} (see \cref{sec:cacnphard}).
On the other hand, a version with multiple conditioning sets (one for each $s_i\in S$) would require additional algorithmic breakthroughs, due to its similarity to the as-yet unsolved gIS.

\subsection{Efficient Algorithm for Finding Instrumental Cutsets}
\label{sec:algo}

To demonstrate efficiency of IC, we develop a polynomial-time algorithm that finds all structural parameters identifiable through iterative application of \cref{thm:unconditionedcutset}.
To do so, we show that the conditions required by \cref{thm:unconditionedcutset} can be reduced to 
finding a match-block in $G_{aux}$:

\begin{restatable}{theorem}{cutsetsubset}
    \label{thm:cutsetsubset}
    Given directed graph $G=(V,D)$, a target edge $x\rightarrow y$,  a set of ``candidate sources" $S$, and the vertex min-cut $C$ between $S$ and $Pa(y)$ closest to $Pa(y)$,
    then there exist subsets $S_f\subseteq S$ and $T_f\subseteq Pa(y)$ where $|S_f|=|T_f|+1=k$ such that
    \begin{enumerate}
        \item  the max-flow from $S_f$ to $T_f\cup \{x\}$ is $k$ in $G$, and
        \item the max-flow from $S_f$ to $T_f\cup\{y\}$ in $G'$ where $x\rightarrow y$ is removed is $k-1$ 
    \end{enumerate}
    if and only if $x$ is part of a match-block between $C$ and $Pa(y)$ in $G$ with all edges incoming to $c_i\in C$ removed. 
\end{restatable}

Note that the ``closest min-cut'' $C$ required by \cref{thm:cutsetsubset} can be found using the Ford-Fulkerson algorithm with $Pa(y)$ as source and $S$ as sink \citep{picardStructureAllMinimum1982}.

Theorem \ref{thm:cutsetsubset} was proven by explicitly constructing the sets $S_f$ and $T_f$ using a match-block. The procedure for doing so is given in \cref{alg:ic}.
It works by finding a set $S_f$ which has a full flow to $C$, which in turn has a match-block to $Pa(y)$ (due to the requirement that none of the $S_f$ have paths to $y$ through $Sib(y)\leftrightarrow y$). 
The min-cut ensures that once $x\rightarrow y$ is removed, all paths to $y$ must still go through the set $C$, and the match-block from $C$ to $Pa(y)$ ensures that there is no way to reorder the paths to create
a flow to $y$ through a different parent. This guarantees that the flow constraints are satisfied, so there is a corresponding IC. 
The full algorithm for finding all edges identifiable with ICs can be constructed by recursively applying the procedure on the auxiliary flow graph, as shown in \cref{alg:icid} (ICID)\footnote{A Python implementation is available at \url{https://github.com/dkumor/instrumental-cutsets}}.

\begin{algorithm}[t]
    \caption{\textsc{IC} solves for edges incoming to $y$ given a set of known edges $\Lambda^*$ }
    \label{alg:ic}
    \begin{algorithmic}
        \Function{IC}{$G,y,\Lambda^*$}
        \State $G_{aux} \leftarrow \textsc{AuxiliaryFlowGraph}(G,\Lambda^*)$
        \State $T\leftarrow $ all sink-node parents of $y'^*$ in $G_{aux}$
        \State $G_{aux}^y \leftarrow G_{aux}$ with edges $t_i\in T$ to $y'^*$ removed
        \State $S \leftarrow $ Source nodes in $G_{aux}^y$ which are not ancestors of $y'^*$
        \State $C \leftarrow \textsc{ClosestMinVertexCut}(G_{aux},S,T)$
        \State $S_f \leftarrow$ elements of $S$ that have a full flow to $C$
        \State $(C_m,T_m) \leftarrow \textsc{MaxMatchBlock}(G_{aux} \text{with edges to $C$ removed},C,T)$
        \State $T_f \leftarrow $ elements of $T$ that are part of a full flow between $C\setminus C_m$ and $T\setminus T_m$
        \State\Return $(S_f,T_f \cup T_m,T_m)$
        \EndFunction
    \end{algorithmic}
\end{algorithm}

\section{Conclusion}

We have developed a new, polynomial-time algorithm for identification in linear SCMs. Previous algorithms with similar identification power had either exponential 
or unknown complexity, with existing implementations using exponential components. 
 Finally, we also showed that the promising method called tsIV cannot handle arbitrarily large graphs due to its inherent computational complexity.

\section*{Acknowledgements}
Bareinboim and Kumor are supported in parts by grants from NSF IIS-1704352, IIS-1750807 (CAREER), IBM Research, and Adobe Research. Part of Chen's contributions were made while at IBM Research.

\bibliography{bibliography.bib}
\bibliographystyle{icml2018}

\newpage

\appendix

\section{Appendix}

This section contains the proofs of theorems mentioned in the main text, as well as efficient versions of 
algorithms for instrumental sets and sets of auxiliary variables. We also include a brief discussion of the difficulties of adding conditioning to instrumental cutsets.

\subsection{Definitions \& Theorems from Previous Works}
\label{sec:prevdef}
The proofs will make extensive use of a couple important results in the literature, which are stated here.
Also given are the full definitions of certain concepts which were only briefly mentioned in the main text due to space constraints.

\begin{definition}{\citep{sullivantTrekSeparationGaussian2010}}
A path $\pi$ from nodes $v$ to $w$ is a \textbf{trek} if it has no colliding arrowheads, that is, $\pi$ is of the form:
$$
v \leftarrow ... \leftarrow \ \leftrightarrow \ \rightarrow ... \rightarrow w \hspace{1cm}
v \leftarrow ... \leftarrow k \rightarrow ... \rightarrow w\hspace{1cm}
v \leftarrow ... \leftarrow w\hspace{1cm}
v \rightarrow ... \rightarrow w
$$
\end{definition}
\begin{definition}{\citep{foygelHalftrekCriterionGeneric2012}}
    A \textbf{half-trek} from $v$, is a trek which either starts from a bidirected edge ($v \leftrightarrow ... \rightarrow w$) or is a directed path ($v \rightarrow ... \rightarrow w$)
\end{definition}
\begin{definition}{\citep{sullivantTrekSeparationGaussian2010}}
    A \textbf{trek monomial} $\pi(\Lambda,\Omega)$ for trek $\pi$ is defined as the product of the structural parameters along the trek, multiplied by the trek's top error term covariance.
\end{definition}
We define $tr(v)$ and $htr(v)$ as the sets of nodes reachable from $v$ with a trek and half-trek respectively. There are two variants of trek monomial, one where $\pi$ does not take a bidirected edge, 
and instead has a top node $k$\footnote{Note also that we can have a trek from $v$ to $v$, including a trek that takes no edges at all, which would be simply $\epsilon_{vv}$} (1), 
and one where the trek contains a bidirected edge $\epsilon_{ab}$ (2).
We can write the covariance between $v$ and $w$, $\sigma_{vw}$, as the sum of the trek monomials of all treks between $v$ and $w$ ($\mathcal T_{vw}$) (3):
$$(1)\hspace{0.7em}\pi(\Lambda,\Omega) = \epsilon^2_k \prod_{x\rightarrow y \in \pi} \lambda_{xy} \hspace{0.7cm} (2)\hspace{0.7em}\pi(\Lambda,\Omega) = \epsilon_{ab}\prod_{x\rightarrow y \in \pi} \lambda_{xy} \hspace{0.7cm} (3)\hspace{0.7em} \sigma_{vw} = \sum_{\pi\in \mathcal T_{vw}} \pi(\Lambda,\Omega)$$

We can reason about the determinants of covariance matrix minors by looking at the flow graph:
\begin{definition}{\citep{sullivantTrekSeparationGaussian2010,foygelHalftrekCriterionGeneric2012,weihsDeterminantalGeneralizationsInstrumental2017}}
The flow graph of $G=(V,D,B)$ is the graph with vertices $V\cup V'$ containing the edges:
\begin{itemize}
    \item $j\rightarrow i$ and $i'\rightarrow j'$ with weight $\lambda_{ij}$ if $i\rightarrow j\in V$
    \item $i\rightarrow i'$ with weight $\epsilon_{ii}$ for all $i\in V$
    \item $i\rightarrow j'$ with weight $\epsilon_{ij}$ if $i\leftrightarrow j \in B$
\end{itemize}
This graph is referred to as $G_{flow}$. The nodes without $'$ are called ``source'', or ``top'' nodes, and the nodes with $'$ are called ``sink'' or ``bottom'' nodes.
\end{definition}

\begin{lemma}{(Gessel-Viennot-Lindstr\"{o}m \cite{gesselDeterminantsPathsPlane1989,sullivantTrekSeparationGaussian2010,draismaPositivityGaussianGraphical2012})}
\label{lemma:gvl}
Given DAG $G=(V,E)$, with $E$ defined as a weighted adjacency matrix, and vertex sets $A,B\subseteq V$, where $|A|=|B|=l$,
$$
\det\left[(I-E)^{-1}\right]_{A,B} = \sum_{P\in N(A,B)} (-1)^P\prod_{p\in P}e^{(p)}
$$
Here, $N(A,B)$ is the the set of all collections of nonintersecting systems of $l$ directed paths in $G$ from $A$ to $B$.
$P=(p_1,...,p_l)$ is a collection of these nonintersecting paths.
$e^{(p)}$ is the product of the coefficients of $e\in E$ along path $p$,
and $(-1)^P$ is the sign of the induced permutation of elements from $A$ to $B$.
\end{lemma}
The above lemma means that a covariance matrix minor's determinant can be found by looking at nonintersecting path sets between top and bottom nodes in the flow graph (\cref{fig:iv_flow}).
We will use $\Sigma_{S,T}$ to represent the covariance matrix minor with rows of $S$ and columns of $T$. While the flow graph has repeated edge weights, this does not affect the rank of the minor:
\begin{lemma}{(\citet{weihsDeterminantalGeneralizationsInstrumental2017} corrolary 3.3)}
    \label{lemma:fullrank}
Let $S=\{s_1,...,s_k\},T=\{t_1,...,t_m\}\subset X$, then $\Sigma_{S,T}$ has generic rank $r$ if and only if the max-flow from $s_1,...,s_k$ to $t_1',...,t_m'$ in $G_{flow}$ is $r$
\end{lemma}
We can therefore say the determinant of a covariance minor from $S$ to $T$ is nonzero iff there is a full nonintersecting path set from $S$ to $T$ in the flow graph.

\subsubsection{Existing Identification Criterions}

This subsection contains an unordered list of identification criteria we reference when developing our algorithms.

\begin{definition} \citep{britoGeneralizedInstrumentalVariables2002}
\label{def:gIS}
The set $Z$ is said to be a \textbf{generalized instrumental set} relative to $X$ and $y$ in $G$ if there exists a set $Z\subset V$ with $|Z|=|X|=k$ and 
set $W=\{W_{z_1},...,W_{z_k}\}$ with all $W_{z_i}\subset V\setminus De(y)$ such that
\begin{enumerate}
    \item There is a path set $\Pi=\{\pi_1,...,\pi_k\}$ without sided intersection between the $Z$ and $X$,
    \item $W_{z_i}$ d-separates $z_i$ from $y$ in $G$ with edges $X\rightarrow y$ removed, but does not block the path from $\Pi$ starting from $z_i$
\end{enumerate}
\end{definition}

\begin{theorem} \citep{britoGeneralizedInstrumentalVariables2002}
\label{thm:gISid}
If there exists a generalized instrumental set to $X$, then the structural parameters $\lambda_{x_iy}$ are identifiable for $x_i\in X$.
\end{theorem}

A simple conditional instrumental set is defined as:
\begin{definition}\citep{vanderzanderSearchingGeneralizedInstrumental2016}
    \label{def:scIS}
The set $Z$ is said to be a \textbf{simple conditional instrumental set} relative to $X$ and $y$ in $G$ if there exist sets $Z\subset V$ and $W\subset V$ such that:
\begin{enumerate}
\item There exists a set of paths $\Pi=\{\pi_1,...,\pi_k\}$ of size $k$ without sided intersection from $Z$ to $X$
\item $W$ d-separates all $Z$ from $y$ in $G$ with the edges $X\rightarrow y$ removed, but does not block any path in $\Pi$
\end{enumerate}
\end{definition}

\begin{lemma}
If a simple conditional instrumental set exists for a set $X$, then all parameters $\lambda_{x_1y}$ with $x_i\in X$ are identifiable.
\end{lemma}
\begin{proof}
If there exists a simple conditional instrumental set, can directly construct a generalized instrumental set (\cref{def:gIS}), which makes the parameters identifiable by \cref{thm:gISid}.
\end{proof}

Auxiliary variables were defined as:
\begin{definition}\citep{chenIdentificationModelTesting2017}
    \label{def:auxvar}
    Let $M=(\Lambda,\Omega)$ be a linear SCM with variables $X$, and $\Lambda^*$ be a set of identified structural parameters.
    An \textbf{auxiliary variable} for $x_i\in X$ is defined as
    $
        x_i^* = x_i - \sum_{\lambda_{x_jx_i}\in \Lambda^*} \lambda_{x_jx_i} x_j
    $.
    \end{definition}
    
An auxiliary instrumental set is defined as
\begin{definition}\citep{chenIncorporating2016}
\label{def:AVS}
A Markovian linear SCM with graph $G$ and set of directed edges $\Lambda^*$ whose coefficient values are known is
known as an auxiliary instrumental set for edges $E\subseteq Pa(y)$ if $Z^*$ is an instrumental set for $E$ where any paths through edges incoming to $z_i\in Z$ through edge $\lambda_{wz_i}\in \Lambda^*$ are considered blocked.
\end{definition}
The corresponding conditional auxiliary variable has the following definition:
\begin{definition}\citep{chenIdentificationModelTesting2017}
\label{def:cAV}
Given graph $G$ and set of directed edges $\Lambda^*$ whose coefficient values are known, a variable $z$ is called
a conditional auxiliary instrument relative to $\lambda_{xy}$, if $z$ is a conditional instrument for $\lambda_{xy}$
in the graph with edges $w_i\rightarrow z$ and $x_i\rightarrow y$ removed where $\lambda_{w_iz},\lambda_{x_iy} \in \Lambda^*$, and no element of conditioning set $W$ is a descendant of $z$.
\end{definition}

A compressed definition of tsIV was given in \cref{def:tsiv}. We include the full version here for completeness:
\begin{theorem}\citep{weihsDeterminantalGeneralizationsInstrumental2017}
    \label{thm:div}
Let $G=(V,D,B)$ be a mixed graph, $w_0\rightarrow v \in G$, and suppose that the edges $w_1\rightarrow v,...,w_l\rightarrow v \in G$ are known to be generically (rationally) identifiable.
Let $G^*_{flow}$ be $G_{flow}$ with the edges $w_0'\rightarrow v',...,w_l'\rightarrow v'$ removed. Suppose there are sets $S\subset V$ and $T\subset V \setminus \{v,w_0\}$ such that $|S|=|T|+1=k$ and
\begin{enumerate}
    \item $De(v)\cap (T\cup \{v\}) = \emptyset$,
    \item the max-flow from $S$ to $T'\cup \{w'_0\}$ in $G_{flow}$ equals $k$, and
    \item the max-flow from $S$ to $T'\cup \{v'\}$ in $G^*_{flow}$ is $<k$,
\end{enumerate}
then $w_0\rightarrow v$ is rationally identifiable by the equation
$$
\lambda_{w_0v} = \frac{\left|\Sigma_{S,T\cup \{v\}}\right| - \sum_{i=1}^l \left|\Sigma_{S,T\cup \{w_i\}}\right|}{\left|\Sigma_{S,T\cup \{w_0\}}\right|}
$$
\end{theorem}

\subsection{Match-Blocking}

The matchblock definition is restated here for convenience:

\matchblock*

The algorithm for finding a match-block (\cref{alg:matchblock}) relies on the following theorem:

\subsets*
\begin{proof}
Suppose not. That is, suppose that $\exists t_i\in T$ such that a max-flow $\mathcal{F}$ gives 0 flow through it, 
yet there exist subsets $S_m\subseteq S,T_m\subseteq T$ with  $t_i\in T_m$, such that $De(S_m)\cap T \subseteq T_m$, and there is a full flow $F_m$ from $S_m$ to $T_m$ in the graph where each vertex has capacity 1 ($|S_m|=|T_m|=|\mathcal{F}_m|$).

We reason about the intersection of the two flows $\mathcal{F}$ and $\mathcal{F}_m$. Our goal is to show that we can modify $\mathcal{F}$ to include the paths of $\mathcal{F}_m$, 
without intersecting paths in $\mathcal{F}$ to $T\setminus T_m$,
thus creating a new flow larger than $\mathcal{F}$ - which is a contradiction, since $\mathcal{F}$ is a max-flow.

Suppose that $\mathcal{F}$ has non-zero flow from $S_f\subseteq S$ to $T_f\subset T$. For any $t_j \in T_m\cap T_f$ we can simply remove the original path in $\mathcal{F}$ to $t_j$, 
resulting in a new flow $\mathcal{F}'$, which is of size $|\mathcal{F}|-|T_m\cap T_f|$. 

Next, we show that the flow paths in $\mathcal{F}_m$ cannot intersect with any flow paths from $\mathcal{F}'$. This is because $De(S_m)\cap T \subseteq T_m$ requires that all matching descendants of elements in $S_m$ are in $T_m$,
meaning that any flow-path $p_k\in \mathcal{F}'$ intersecting flow-path $p_l' \in F_m$ must have an endpoint on an element in $T_m$ - but all such paths were removed in $\mathcal{F}'$.

We combine the paths of $\mathcal{F}'$ and $\mathcal{F}_m$, giving a new flow $\mathcal{F}''$ from $S$ to $T$. Finally, since $t_i \notin T_f$, $|T_m\cap T_f| <|T_m|$, so:

$$
\begin{aligned}
|\mathcal{F}|-|T_m\cap T_f| + |T_m| &= |\mathcal{F}''|\\
|\mathcal{F}|&< |\mathcal{F}''|\\
\end{aligned}
$$

A contradiction ($\mathcal{F}$ is a max-flow). There cannot be any satisfied subset $T_m$ containing $t_i$.

Finally, no ancestor of $t_i$ can be part of a satisfied subset, since they all have paths to $t_i$, which completes the proof.
\end{proof}

\begin{restatable}{corollary}{polymatchblock}
    \label{cor:polymatchblock}
    Given directed acyclic graph $G=(V,D)$ and sets of source nodes $S$,$T$, their maximal match-blocked subsets $S_m$,$T_m$ can be found in polynomial time.
\end{restatable}

\begin{proof}
See \cref{alg:matchblock}. At each iteration, at least one node is eliminated from the feasibility set, meaning that given $n$ nodes, the algorithm runs at most $n$ max-flow queries, each of which is computable in polynomial time.
\end{proof}

\begin{restatable}{corollary}{maximalmatchblock}
\label{cor:maximalmatchblock}
Suppose that $S_m$,$T_m$ constitute a match-block. The match-block $S'_m,T'_m$ found using \cref{alg:matchblock} is such that $S_m \subseteq S'_m$ and $T_m\subseteq T'_m$.
\end{restatable}
\begin{proof}
At each step of the algorithm, only nodes $t_i,...t_j$ that have 0 flow crossing through them are removed from $T$. These nodes cannot be part of $T_m$ by \cref{thm:cutsetsubset}. Similarly, only ancestors to $t_i,...,t_j$ are removed from $S$, which likewise cannot be part of $S_m$. Therefore, no elements of $S_m$ nor $T_m$ were removed by the algorithm, meaning that $S_m \subseteq S'_m$ and $T_m\subseteq T'_m$.
\end{proof}

\begin{restatable}{corollary}{ivsubsets}
    \label{cor:ivsubsets}
    Given mixed graph $G=(V,D,B)$, there exists a valid instrumental subset $S_m \subseteq V \setminus De(Sib(y))$ to $T_m\subseteq Pa(y)$ if and only if $S_m,T_m$ is a match-block 
    between $V \setminus De(Sib(y))$ and $Pa(y)$ in $G_{flow}$. 
\end{restatable}
\begin{proof}
    $\Rightarrow$: Given a valid instrumental subset, it is also a match-block, because there are an equal number of instruments and parents of $y$, and there exists a system of nonintersecting paths from $S_m$ to $T_m$. Furthermore, if any of the $S_m$ has paths to $Pa(y)\setminus T_m$,
    then they are dependent on $y$ in the graph with the edges $T_m\rightarrow y$ removed, which means that they are not an IS. Finally, the IS can't have elements of $De(Sib(y))$, since it would mean that $y$ is not-d-separated from the instruments.
    As such, the sets satisfy the requirements of a match-block.

    $\Leftarrow$: Suppose there is a valid match-block, then there is a valid instrumental set. The match-block guarantees the requirements of an IS directly. The restriction of $S$ to elements without back-paths to $y$ forces all paths to $y$ to go through $Pa(y)$.

The corresponding algorithm is given in \cref{alg:IS}
\end{proof}

\begin{algorithm}
    \caption{Find Maximal Instrumental Subsets given graph $G$, target variable $y$}
    \label{alg:IS}
    \begin{algorithmic}
        \Function{MaxIS}{G,y}
        \State $G_{flow} \leftarrow \textsc{FlowGraph}(G)$
        \State $Z\leftarrow (An(y,G_{flow})\cap \textsc{SourceNodes}(G_{flow})) \setminus \textsc{SourceNodesOf}(De(Sib(y)))$
        \State $(Z_f,X_f) \leftarrow \textsc{MaxMatchBlock}(G_{flow},Z,Pa(y)')$
        \State\Return $(Z_f,X_f)$
        \EndFunction
        \end{algorithmic}
\end{algorithm}

\subsubsection{Auxiliary Flow Graph}

This section develops results that show the auxiliary flow graph can be used instead of $G_{flow}$, and that it encodes treks through auxiliary variables. For reference, $G_{aux}$ is defined as:

\auxflowgraph*

\begin{lemma}
    \label{lemma:auxflowcov} 
Given a linear SCM $(\Lambda,\Omega)$ with causal graph $G=(V,D,B)$, a set of known structural parameters $\Lambda^*$, and defining $V^*=\{v_1^*,...,v_k^*\}$ as
$$
v_i^* = v_i - \sum_{\lambda_{v_jv_i}\in \Lambda^*} \lambda_{v_jv_i} v_j
$$
the sum over each path of products of weights along the path from $s\in V\cup V^*$ to $t\in V'$ in the auxiliary flow graph $G_{aux}$ encodes the covariance $\sigma_{st}$.
\end{lemma}
\begin{proof}
We already know that $G_{flow}$ encodes treks in the graph \citep{sullivantTrekSeparationGaussian2010}. If $s\in V$, we notice that the sum of paths can be constructed 
by combining the paths from $s$ that are not passing $s^*$, and the paths from $s^*$ multiplied by 1 - which results in the treks, identically to $G_{flow}$.

If $s^*\in V^*$, we notice that the set of treks from $s^*$ to a variable $y$ can be seen as a subset of the treks from $s$
$$
\begin{aligned}
\sigma_{sy} &= \left(\text{treks not starting from any $\lambda_{a_js}\in \Lambda^*$}\right) + \left(\text{treks starting from one of the $\lambda_{a_js}\in \Lambda^*$}\right) \\
\sigma_{sy} &= \left(\text{treks not starting from any $\lambda_{a_js}\in \Lambda^*$}\right)  + \sum_j \lambda_{a_js}\sigma_{a_jy}\\
\sigma_{sy} -  \sum_j \lambda_{a_js_i}\sigma_{a_jy} &= \left(\text{treks not starting from any $\lambda_{a_js}\in \Lambda^*$}\right) \\
\sigma_{s^*y} &= \left(\text{treks not starting from any $\lambda_{a_js}\in \Lambda^*$}\right) \\
\end{aligned}
$$

Using this result, we can conclude that the covariance of $s^*_i$ with any variable behaves as if the edges from $s$ to $a_j$ in $G_{flow}$ did not exist, but all other paths were identical to $G_{flow}$.
This is exactly the construction given in $G_{aux}$.
\end{proof}

Next, we show directly that the Gessel-Viennot-Lindrstr\"om Lemma still holds for the auxiliary graph.
While the statement is lengthy, it simply states that we can just use the nonintersecting path sets in the new graph to determine values 
of determinants of minors of the covariance matrix where each variable also has an ``auxiliary'' version of itself, where known effects are removed.
\begin{lemma}{(Auxiliary Gessel-Viennot-Lindstr\"{o}m)}
    \label{lemma:auxgvl}
    Given a linear SCM $(\Lambda,\Omega)$ with causal graph $G=(V,D,B)$, a set of known structural parameters $\Lambda^*$, and defining $V^*=\{v_1^*,...,v_k^*\}$ as
    $
    v_i^* = v_i - \sum_{\lambda_{v_jv_i}\in \Lambda^*} \lambda_{v_jv_i} v_j
    $, for any vertex sets $A\subseteq V\cup V^*$ and $B\subseteq V$, where $|A|=|B|=l$,
    $$
    \det\Sigma_{A,B} = \sum_{P\in N(A,B)} (-1)^P\prod_{p\in P}e^{(p)}
    $$
    Here, $N(A,B)$ is the the set of all collections of nonintersecting systems of $l$ directed paths in the auxiliary flow graph $G_{aux}$ from $A$ to $B$.
    $P=(p_1,...,p_l)$ is a collection of these nonintersecting paths.
    $e^{(p)}$ is the product of the weights along path $p$,
    and $(-1)^P$ is the sign of the induced permutation of elements from $A$ to $B$.
\end{lemma}
\begin{proof}
This is a direct consequence of \cref{lemma:gvl} and \cref{lemma:auxflowcov}.
\end{proof}

\begin{restatable}{corollary}{avspolynomial}
Auxiliary Instrumental Sets can be found in polynomial time
\end{restatable}
\begin{proof}
The proof is identical to \cref{cor:polymatchblock,cor:ivsubsets}, with the only difference being that the auxiliary flow graph is
used in place of $G_{flow}$. The full algorithm is shown in \cref{alg:AVS}.
\end{proof}

\begin{algorithm}
    \caption{Finds all edges identifiable using AVS in polynomial time}
    \label{alg:AVS}
    \begin{algorithmic}
        \Function{AVS}{G,y,$\Lambda^*$}
        \State $G_{aux}\leftarrow \textsc{AuxiliaryFlowGraph}(G,\Lambda^*)$
        \State $T\leftarrow $ all sink-node parents of $y'^*$ in $G_{aux}$
        \State $G_{aux}^y \leftarrow G_{aux}$ with edges $t_i\in T$ to $y'^*$ removed
        \State $S \leftarrow $ Source nodes in $G_{aux}^y$ which are not ancestors of $y'^*$
        \State \Return $\textsc{MaxMatchBlock}(G_{aux},S,T)$
        \EndFunction
    \end{algorithmic}
    \begin{algorithmic}
        \Function{AVSID}{G}
        \State $\Lambda^* \leftarrow \emptyset$
        \Do
        \ForAll{$y\in G$}
        \State $(\_,T_m) \leftarrow \textsc{AVS}(G,y,\Lambda^*)$
        \State $\Lambda^* \leftarrow \Lambda^*\cup \{\lambda_{ty} | t\in T_m\}$
        \EndFor
        \DoWhile {at least one parameter was identified in this iteration}
        \State\Return $\Lambda^*$
        \EndFunction
        \end{algorithmic}
\end{algorithm}

\newpage
\subsection{NP-Hardness of tsIV and scIV}

We first relate scIV to tsIV, then we construct a supporting lemma used within our main NP-hardness proof,
and finally, we prove that tsIVs and scIVs are NP-Hard, and therefore NP-Complete.

\begin{theorem}
    \label{thm:scIS2div}
    If there exists a simple conditional instrumental set usable to solve for $\lambda_{xy}$, then there exists a tsIV that can be used to solve for $\lambda_{xy}$
\end{theorem}

\begin{proof}
Suppose there is a simple instrumental set described by $Z,X,W$, where $W$ is the conditioning, $X$ is a set of $y$'s parents, and $Z$ is the set of instruments. We claim that the corresponding tsIV has $S=Z\cup W$ and $T=W\cup X\setminus\{x\}$ for any $x\in X$.

To witness, observe that condition 1 of \cref{thm:div} is satisfied, since $X$ and $W$ are non-descendants of $y$, and condition 2 is also satisfied, since we can construct a full flow
between $S$ and $T\cup\{x\}$ by adding a path from each $w_i$ to $w'_i$, and using the paths $\Pi$ from \cref{def:scIS} from $Z$ to $X$. If a path from $z$ to $x$ crosses a collider $w_i$, we 
construct the path $z$ to $w_i'$ and $w_i$ continuing on the path to $x'$.

Finally, we can focus on condition 3 of \cref{thm:div}. We will prove it holds by contradiction. 
Suppose that there exists a full flow from $S$ to $T\cup\{y\}$ in the flow graph with edge $x'\rightarrow y'$ removed.

We know $W$ d-separates $Z$ from $y$ when the edges $\lambda_{x_iy}, x_i \in X$ are removed. Since $T$ contains all $x_i\in X$ except $x$ itself, none of the edges $x_i'\rightarrow y'$ can be taken,
since their corresponding vertex $x_i'$ is already part of a path. Likewise, since $x'\rightarrow y'$ is removed, none of the paths can take that edge either.

Nevertheless, the $S$ must have a valid matching to $T\cup\{y\}$ for a full flow to exist. We now show that this is impossible by induction.

Let $s_0\in S$ be the element matched to $y'$ in the full flow. 
Either $s_0\in Z$ or $s_0\in W$.

We know that $s_1 \notin Z$, since by d-separation,
    all treks from $z_i$ to $y'$ that don't pass the removed edges are intersected by elements of $W$, which corresponds to blocking both the top and bottom nodes of the flow graph.
This means that $s_1\in W$.

Suppose that for $s_1,...,s_i$, each $s_k \in W$. 
Let $s_{i+1} \in S$ be matched to $s'_i$ in the full flow. 
Suppose $s_{i+1} \in Z$. Then it means that there is a path $s_{i+1}, s_i,...,s_1,y$ across v-structures to $y$, making the element of $z$ not d-separated from $y$,
a contradiction. Therefore $s_{i+1}$ must be in $W$. But there is a finite number of $W$, meaning that the only possible elements to match to $s_n'$ after all $W$ are already matched to something will be elements of $Z$,
a contradiction.

This proof showed that any full matching of $S$ to $T\cup \{y\}$  has a confounding path across v-structures, which means that $z_i$ was not d-separated from $y$, a contradiction.
\end{proof}

\begin{lemma}
\label{lem:nobidirected}
Given sets $S$ and $T$, and a full flow $\mathcal{F}_x$ of size $k$ between $S$ and $T\cup\{x\}$,
if $\mathcal{F}_x$ has an $s_i\in S$ matched to $x$, 
then if the bidirected edge $s_i\leftrightarrow y$ exists, there exists a flow of size $k$ between $S$ and $T\cup \{y\}$, 
and there does not exist a tsIV for $\lambda_{xy}$.
\end{lemma}
\begin{proof}
    We can construct a flow $\mathcal{F}_y$ of size $k$ by replacing the path $s_i$ to $x$
    with $s_i \leftrightarrow y$, which is a flow of size $k$ from $S$ to $T\cup\{y\}$, and means that no tsIV exists.
    
\end{proof}

\divsat*
\begin{proof}
First, we convert the formula into a version without repeated literals in any clause by removing variables of repeated literals from all clauses where they appear (a repeated literal must be false, since otherwise the clause would have 2 trues, 
and would also force the remaining literal to be true). Similarly, we simplify out formulae with a literal and its negation in a single clause.
We can simplify the formula such that each clause does not have any repeated statements. We operate upon this converted formula.

    Consider $F$ being made up of $k$ clauses $c_1,...,c_k$, where clause $i$ has literals $l_{i1},l_{i2},l_{i3}$.

    \textbf{$\Rightarrow$: We prove that there exists a tsIV in $G$ if $F$ is satisfiable}, by constructing a simple conditional instrumental set (\cref{def:scIS}) for the $X$ variables,
    and using \cref{thm:scIS2div} to show that a corresponding tsIV exists.

    Let $l_{is} \in c_i$ be the true literal of a satisfying assignment. 
    We construct $W = \bigcup_{c_{i}\in F} \{w_{is}\}$ and $Z = \bigcup_{c_{i} \in F} \{z_{is}\}$. We now show that
    these sets describe a simple conditional instrumental set. 
    
    Let $G'$ be the graph where all edges $X\rightarrow y$ are removed. Each $z_i\in Z$ has the correspoding $w_i$
    conditioned. None of the bidirected edges between $z$ and $w$ are between elements of $Z,W$, since that would mean
    that the assignments are incompatible (either two literals true in a single clause, or a literal in a different clause being incompatible 
    with the assignment implied by $l_{is}$). This means that each $z_i\in Z$ is d-separated from $y$ in $G'$, and therefore $Z,W,X$ is a simple instrumental set (\cref{fig:divpg,fig:divph}).

\textbf{$\Leftarrow$: We show that if there does not exist a valid assignment to $F$, then there do not exist sets $S$ and $T$ that can be used as a tsIV for $\lambda_{x_1y}$.}
We will exploit conditions 2 and 3 of \cref{thm:div} to show that whenever there is a flow of size $k$ between  $S$ and $T\cup\{x_1\}$, then there is also a flow of size $k$
between $S$ and $T\cup\{y\}$ in the graph with $x_1'\rightarrow y'$ removed.

This is easiest to prove through contradiction.  Suppose that there exists a valid tsIV despite there being no satisfying
1-in-3SAT assignment to $F$. This means that there exists a flow $\mathcal{F}_x$ of size $k$ between $S$ and $T\cup \{x_1\}$ but no flow of size $k$
between $S$ and $T\cup\{y\}$.

Since flows correspond to nonintersecting path sets, then  $\mathcal{F}_x$ represents a matching of nonintersecting paths, one of which is from some $s_i\in S$ to $x_1$. 
We know that this $s_i$ cannot be any of the $X$ or $W$ nodes,
since they have bidirected edges to $y$ in $G$ (\cref{lem:nobidirected}).

This means that the only possibility for generating a valid tsIV is for $s_i\in Z$ to match with $x_1$ in $\mathcal{F}_x$. 

We now know that any valid tsIV has an element $s_i = z_i \in Z$ matched with $x_1$. 
Notice that $z_i$ has paths to $y$ through all of the $x_2',...,x_n'$.
We could construct a flow $\mathcal{F}_y$ e.g. with the path $z_i\rightarrow z_i' \rightarrow x_1'$ replaced with a path $z_i \rightarrow z_i' \rightarrow x_j' \rightarrow y$,
which once again corresponds to a full flow, meaning that no such tsIV exists (see \cref{fig:divp2a,fig:divp2b}).
Since $z_i$ is matched with $x_1$, 
in must have an unblocked path to $x_1$ through one of the $z_j'$, so all $x_i'$ must be in $T$ to disallow constructing such a $\mathcal{F}_y$.

Next, we require that each of the $x_i\in X\setminus\{x_1\}$ added to $T$ has a corresponding matched variable in $S$. None of the $X$ can be used for this function, since they have bidirected edges to $y$,
which could be used to construct a full flow to $y$ as shown in \cref{fig:divp2c,fig:divp2d}. 

Likewise, none of the $W$ can match to these $X$, because, once again, if there was an $x_j'$ matched to a $w_l$,
we could construct a new flow of size $k$, $\mathcal{F}_y$, which has $z_i$ matched with $x_j'$, 
and uses the $w_l\in W$'s bidirected edge to $y$ to create a full flow.

This means that we must have at least $k$ elements of $Z$ in $S$. Each of these $z_j \in S\cap Z$ which is matched to elements of $X$ has an open path to $x_1$, meaning that 
all of them need to have $w_j\in S$ to block the back-path from $z_j$ that could match with $y$ through $w$'s bidirected edge.

At this point, we have shown that any tsIV in a graph constructed as given in \cref{thm:divsat} must have in $S$ a set of $k$ $Z$ variables, called $Z_k$ which all have unblocked paths to the $X$,
and a set $W_k=W\cap Pa(Z_k)$, meaning $Z_k\cup W_k \subseteq S$ and $X\setminus\{x\} \subset T$.

Next, we need to add nodes that match to $W_k$ to $T$, 
since the flow must be of full size from $S$ to $T\cup \{x\}$. 

We first claim that the corresponding nodes cannot be elements of $Z_k'$.
Suppose not, that is, suppose that $\exists w_i\in W_k$ matched to $z_i'$.
The only possible way for this to be true is for $z_i$'s path to $x_i$ to be as shown in \cref{fig:divpe,fig:divpf}.
However, any such matching can be flipped so that it is $w_i$ matching to $x_j'$ and $z_i$ to $z_i'$, which would mean the existence of a full flow, and therefore no tsIV (\cref{lem:nobidirected}).

Furthermore, building upon this result, we will claim that any $z_i$ matched with $x_j$ must have a valid matching through $z_i \rightarrow z_i' \rightarrow x_j'$.
That is, we claim that $z_i'$ cannot be blocked. This can be seen by contradiction.
Suppose $z_i'\in T$. Then there must be a $w_j$, $z_j$, or even $x_j$ which matches with $z_i'$ ($i\ne j$).
However, we can then take the flow $\mathcal{F}_x$, and match $z_i$ with $z_i'$, and have the node originally
matched with $z_i'$ take the bidirected edge $w_j\leftrightarrow y$ instead.
This means we can construct a full flow $\mathcal{F}_y$, meaning that a tsIV cannot exist.
We can therefore assume that all $k$ matchings from the $z_i$ to $x_j$ go through the nodes $z_i'$ instead of taking bidirected edges 
(if $\mathcal{F}_x$ has a flow from $z_i$ through a bidirected edge to $x_j$, can replace it with the flow through $z_i'$ to create another full flow $\mathcal{F}_x'$).

\textbf{Recap:} We know that the $k$ nodes that match with the $k$ $x$ nodes must all be $z$ nodes, and each of those nodes $z_i$ must have its associated $w_i\in S$ to block the possible path to $y$ (\cref{fig:divp3a}). Finally, the $w_i$ cannot be matched to $z_i'$,
and $z_i$ must have its matching path be $z_i\rightarrow z_i' \rightarrow x_j'$.

\textbf{Claim:} We now claim that if a tsIV exists, with sets $S$ and $T$, 
then it must have a full flow where all $k$ $z$ nodes matching to the $x$ have 
their corresponding $w_i$ matched to $w_i'$, 
meaning that the $w_i$ does not match through a bidirected edge, 
but rather matches to itself. That is, we claim that there exists a flow that has $k$ matching substructures of the form
$z_i\rightarrow z_i'\rightarrow x_i'$ and $w_i\rightarrow w_i'$ (\cref{fig:divpc,fig:divpd}). We will call these ``active literals''.

We will prove that there must be $k$ active literals in this flow by contradiction. Suppose not. That is, suppose that the maximal number of active literals is $m<k$. Choose the flow $\mathcal{F}$ with all $m$ active literals. 
Next, choose one of the remaining $z_j$ that is matched to $x_j$,
which is not part of an active literal. $z_j$ must have $w_j\in S$, because the path in \cref{fig:divp3b} must be blocked. The $w_j$ must be matched to an element other than $w_j'$, since it would be an active literal, a contradiction.
The matched edge must be a descendant of $w_i$ in $G_{flow}$, 
meaning that the only candidate is $z_j$ with $j\ne i$ ($x_j$ was already matched to the $k$ $z_i$ values (\cref{fig:divp3c}), and $z_i$ was already shown to be impossible).

 We know $z_j$ is not an active literal, so
we have the paths shown in \cref{fig:divp3d}, which means that $w_j'$ must be blocked 
by adding $w_j'$ as a sink (otherwise we could create a full flow $\mathcal{F}_y$). We now observe the source matched to $w_j'$ in $\mathcal{F}_x$.
$w_j$ cannot be the source (\cref{fig:divp3e}), and same is true of $z_j$, 
due to the the same reason: $w_j\leftrightarrow y$ gives a way to match to $y$ (\cref{fig:divp3f}). 
Finally, once again, no value of $x$ can be part of the matching, due to its bidirected edge to $y$ (\cref{fig:divp3g})

This means that the only possible matching is with another $z_m$ connected to $w_j$ with bidirected edge, which would then need to have $w_m$ blocked (\cref{fig:divp3h}), and can't be an active literal, since those are already matched to $x$.
Now, we have returned to the situation in \cref{fig:divp3a}, where the options are using $w_m$ to match, or having a match through a bidirected edge.

The matches in $\mathcal{F}$ can continue through further bidirected edges, but due to the finite amount of literals in the full formula, 
at some point the chain will end with a $w_n$ matched to $w_n'$. We can then reorder this matching to create a new active literal (\cref{fig:divp3i}). 
This means that we have constructed a matching that will have the same total flow as $\mathcal{F}$, but has 1 more active literal, which is a contradiction, since $\mathcal{F}$ already has the maximal number of active literals.

We therefore conclude that there must be $k$ active literals in the tsIV.

\textbf{Recap:} We know that any tsIV has a full flow $\mathcal{F}$ which includes $k$ active literals, meaning that there are $k$ substructures where $z_i\in Z_k$
has $z_i\rightarrow z_i'\rightarrow x_j$ and $w_i\rightarrow w_i'$ are in $\mathcal{F}$

We exploit the knowledge that there is no satisfying 1-in-3SAT assignment to show that there must 
exist at least two of the $z_i,z_j\in Z_k$ and corresponding $w_i,w_j \in W_k$ such that there are bidirected edges $w_i\leftrightarrow z_j$ and $w_j \leftrightarrow z_i$, allowing us to construct a full flow including $y$ as shown in \cref{fig:divp2e,fig:divp2f}.
With this, we have ensured that there is no possible tsIV if there is no satisfying truth assignment where exactly one literal is true in each clause.

\end{proof}

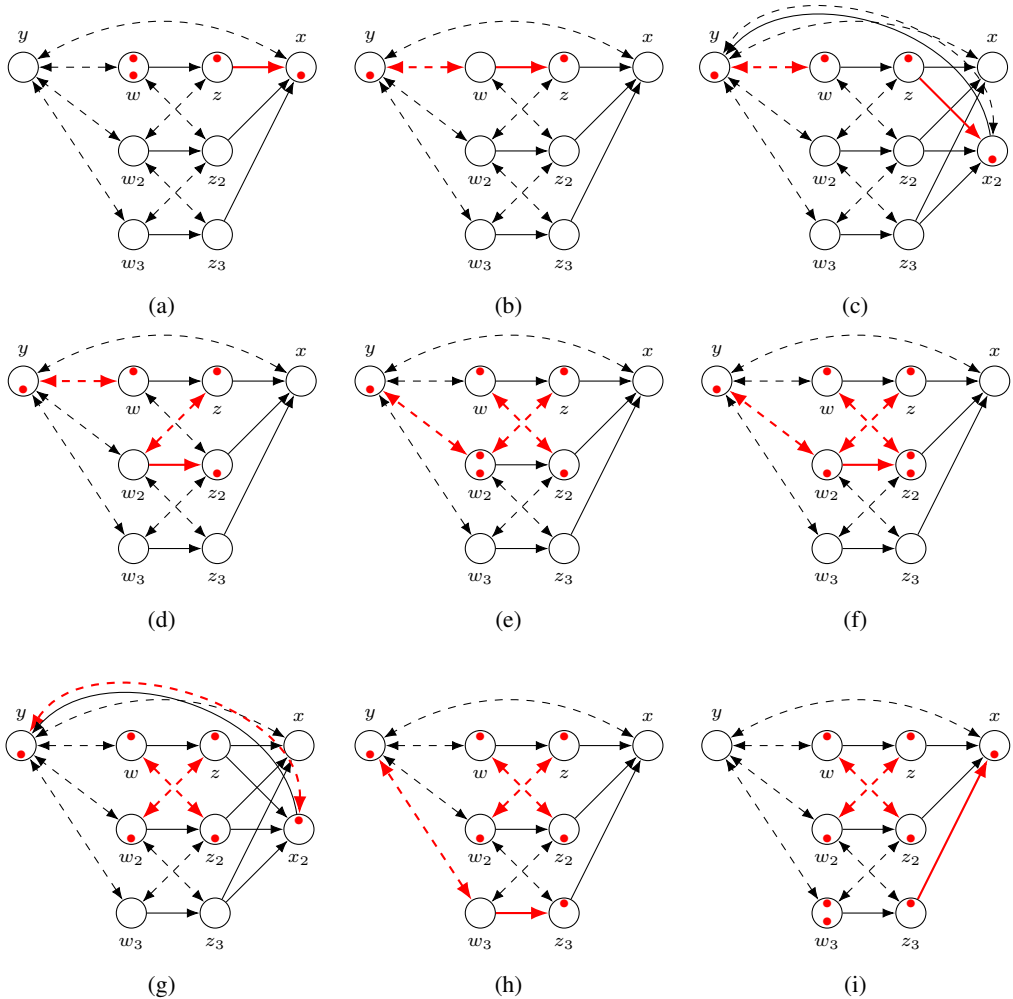
\begin{figure}
    \centering
    \begin{subfigure}[t]{0.33\linewidth}
      \center
      \begin{tikzpicture}[node distance =.7 cm and .7 cm]
  
          \dotnode[rtrb]{w}{}{below:$w$};
          \dotnode[rt]{z}{right = of w}{below:$z$};
          \dotnode[rb]{x}{right = of z}{above:$x$};
          \dotnode[]{w2}{below=of w}{below:$w_2$};
          \dotnode[]{z2}{right = of w2}{below:$z_2$};
          \dotnode[]{w3}{below=of w2}{below:$w_3$};
          \dotnode[]{z3}{right = of w3}{below:$z_3$};
          \dotnode[]{y}{left = of w,xshift=-1em}{above:$y$};
          \path[red,thick] (z) edge (x);
          
          \path (w) edge (z);
  
          \path[bidirected] (x) edge[bend right=30] (y);
          \path[bidirected] (w) edge (y);

          \path (z2) edge (x);
          \path (w2) edge (z2);
  
          \path[bidirected] (w2) edge (y);
  
          \path[bidirected] (w) edge (z2);
          \path[bidirected] (w2) edge (z);

          \path (z3) edge (x);
          \path (w3) edge (z3);
  
          \path[bidirected] (w3) edge (y);
  
          \path[bidirected] (w2) edge (z3);
          \path[bidirected] (w3) edge (z2);
  
      \end{tikzpicture}
  
      \caption{\label{fig:divp3a}}
  \end{subfigure}%
  \begin{subfigure}[t]{0.33\linewidth}
    \center
    \begin{tikzpicture}[node distance =.7 cm and .7 cm]
  
        \dotnode[]{w}{}{below:$w$};
        \dotnode[rt]{z}{right = of w}{below:$z$};
        \dotnode[]{x}{right = of z}{above:$x$};
        \dotnode[]{w2}{below=of w}{below:$w_2$};
        \dotnode[]{z2}{right = of w2}{below:$z_2$};
        \dotnode[]{w3}{below=of w2}{below:$w_3$};
        \dotnode[]{z3}{right = of w3}{below:$z_3$};
        \dotnode[rb]{y}{left = of w,xshift=-1em}{above:$y$};
        \path (z) edge (x);
        
        \path[red,thick] (w) edge (z);
  
        \path[bidirected] (x) edge[bend right=30] (y);
        \path[bidirected,red,thick] (w) edge (y);

        \path (z2) edge (x);
        \path (w2) edge (z2);
  
        \path[bidirected] (w2) edge (y);
  
        \path[bidirected] (w) edge (z2);
        \path[bidirected] (w2) edge (z);

        \path (z3) edge (x);
        \path (w3) edge (z3);
  
        \path[bidirected] (w3) edge (y);
  
        \path[bidirected] (w2) edge (z3);
        \path[bidirected] (w3) edge (z2);
  
    \end{tikzpicture}
  
    \caption{\label{fig:divp3b}}
  \end{subfigure}%
    \begin{subfigure}[t]{0.33\linewidth}
      \center
      \begin{tikzpicture}[node distance =.7 cm and .7 cm]
  
          \dotnode[rt]{w}{}{below:$w$};
          \dotnode[rt]{z}{right = of w}{below:$z$};
          \dotnode[]{x}{right = of z}{above:$x$};
          \dotnode[rb]{x2}{below = of x}{below:$x_2$};
          \dotnode[]{w2}{below=of w}{below:$w_2$};
          \dotnode[]{z2}{right = of w2}{below:$z_2$};
          \dotnode[]{w3}{below=of w2}{below:$w_3$};
          \dotnode[]{z3}{right = of w3}{below:$z_3$};
          \dotnode[rb]{y}{left = of w,xshift=-1em}{above:$y$};
          \path (z) edge (x);
          
          \path (w) edge (z);
  
          \path[bidirected] (x) edge[bend right=30] (y);
          \path[bidirected,red,thick] (w) edge (y);

          \path (z2) edge (x);
          \path (w2) edge (z2);
  
          \path[bidirected] (w2) edge (y);
  
          \path[bidirected] (w) edge (z2);
          \path[bidirected] (w2) edge (z);

          \path (z3) edge (x);
          \path (w3) edge (z3);
  
          \path[bidirected] (w3) edge (y);
  
          \path[bidirected] (w2) edge (z3);
          \path[bidirected] (w3) edge (z2);
  
          \path[red,thick] (z) edge (x2);
          \path (z2) edge (x2);
          \path (z3) edge (x2);
          \path[bidirected] (x2) edge[bend right=75] (y);
          \path (x2) edge[bend right=65] (y);
  
      \end{tikzpicture}
  
      \caption{\label{fig:divp3c}}
  \end{subfigure}%
  
    \begin{subfigure}[t]{0.33\linewidth}
      \center
      \begin{tikzpicture}[node distance =.7 cm and .7 cm]
  
        \dotnode[rt]{w}{}{below:$w$};
        \dotnode[rt]{z}{right = of w}{below:$z$};
        \dotnode[]{x}{right = of z}{above:$x$};
        \dotnode[]{w2}{below=of w}{below:$w_2$};
        \dotnode[rb]{z2}{right = of w2}{below:$z_2$};
        \dotnode[]{w3}{below=of w2}{below:$w_3$};
        \dotnode[]{z3}{right = of w3}{below:$z_3$};
        \dotnode[rb]{y}{left = of w,xshift=-1em}{above:$y$};
          \path (z) edge (x);
          
          \path (w) edge (z);
  
          \path[bidirected] (x) edge[bend right=30] (y);
          \path[bidirected,red,thick] (w) edge (y);

          \path (z2) edge (x);
          \path[red,thick] (w2) edge (z2);
  
          \path[bidirected] (w2) edge (y);
  
          \path[bidirected] (w) edge (z2);
          \path[bidirected,red,thick] (w2) edge (z);

          \path (z3) edge (x);
          \path (w3) edge (z3);
  
          \path[bidirected] (w3) edge (y);
  
          \path[bidirected] (w2) edge (z3);
          \path[bidirected] (w3) edge (z2);
  
      \end{tikzpicture}
  
      \caption{\label{fig:divp3d}}
  \end{subfigure}%
  \begin{subfigure}[t]{0.33\linewidth}
    \center
    \begin{tikzpicture}[node distance =.7 cm and .7 cm]
  
      \dotnode[rt]{w}{}{below:$w$};
      \dotnode[rt]{z}{right = of w}{below:$z$};
      \dotnode[]{x}{right = of z}{above:$x$};
      \dotnode[rtrb]{w2}{below=of w}{below:$w_2$};
      \dotnode[rb]{z2}{right = of w2}{below:$z_2$};
      \dotnode[]{w3}{below=of w2}{below:$w_3$};
      \dotnode[]{z3}{right = of w3}{below:$z_3$};
      \dotnode[rb]{y}{left = of w,xshift=-1em}{above:$y$};
        \path (z) edge (x);
        
        \path (w) edge (z);
  
        \path[bidirected] (x) edge[bend right=30] (y);
        \path[bidirected] (w) edge (y);

        \path (z2) edge (x);
        \path (w2) edge (z2);
  
        \path[bidirected,red,thick] (w2) edge (y);
  
        \path[bidirected,red,thick] (w) edge (z2);
        \path[bidirected,red,thick] (w2) edge (z);

        \path (z3) edge (x);
        \path (w3) edge (z3);
  
        \path[bidirected] (w3) edge (y);
  
        \path[bidirected] (w2) edge (z3);
        \path[bidirected] (w3) edge (z2);
  
    \end{tikzpicture}
  
    \caption{\label{fig:divp3e}}
  \end{subfigure}%
  \begin{subfigure}[t]{0.33\linewidth}
    \center
    \begin{tikzpicture}[node distance =.7 cm and .7 cm]
  
      \dotnode[rt]{w}{}{below:$w$};
      \dotnode[rt]{z}{right = of w}{below:$z$};
      \dotnode[]{x}{right = of z}{above:$x$};
      \dotnode[rb]{w2}{below=of w}{below:$w_2$};
      \dotnode[rtrb]{z2}{right = of w2}{below:$z_2$};
      \dotnode[]{w3}{below=of w2}{below:$w_3$};
      \dotnode[]{z3}{right = of w3}{below:$z_3$};
      \dotnode[rb]{y}{left = of w,xshift=-1em}{above:$y$};
        \path (z) edge (x);
        
        \path (w) edge (z);
  
        \path[bidirected] (x) edge[bend right=30] (y);
        \path[bidirected] (w) edge (y);

        \path (z2) edge (x);
        \path[red,thick] (w2) edge (z2);
  
        \path[bidirected,red,thick] (w2) edge (y);
  
        \path[bidirected,red,thick] (w) edge (z2);
        \path[bidirected,red,thick] (w2) edge (z);

        \path (z3) edge (x);
        \path (w3) edge (z3);
  
        \path[bidirected] (w3) edge (y);
  
        \path[bidirected] (w2) edge (z3);
        \path[bidirected] (w3) edge (z2);
  
    \end{tikzpicture}
  
    \caption{\label{fig:divp3f}}
  \end{subfigure}%
  
  \begin{subfigure}[t]{0.33\linewidth}
    \center
    \begin{tikzpicture}[node distance =.7 cm and .7 cm]
  
        \dotnode[rt]{w}{}{below:$w$};
        \dotnode[rt]{z}{right = of w}{below:$z$};
        \dotnode[]{x}{right = of z}{above:$x$};
        \dotnode[rt]{x2}{below = of x}{below:$x_2$};
        \dotnode[rb]{w2}{below=of w}{below:$w_2$};
        \dotnode[rb]{z2}{right = of w2}{below:$z_2$};
        \dotnode[]{w3}{below=of w2}{below:$w_3$};
        \dotnode[]{z3}{right = of w3}{below:$z_3$};
        \dotnode[rb]{y}{left = of w,xshift=-1em}{above:$y$};
        \path (z) edge (x);
        
        \path (w) edge (z);
  
        \path[bidirected] (x) edge[bend right=30] (y);
        \path[bidirected] (w) edge (y);

        \path (z2) edge (x);
        \path (w2) edge (z2);
  
        \path[bidirected] (w2) edge (y);
  
        \path[bidirected,red,thick] (w) edge (z2);
        \path[bidirected,red,thick] (w2) edge (z);

        \path (z3) edge (x);
        \path (w3) edge (z3);
  
        \path[bidirected] (w3) edge (y);
  
        \path[bidirected] (w2) edge (z3);
        \path[bidirected] (w3) edge (z2);
  
        \path (z) edge (x2);
        \path (z2) edge (x2);
        \path (z3) edge (x2);
        \path[bidirected,red,thick] (x2) edge[bend right=75] (y);
        \path (x2) edge[bend right=65] (y);
  
    \end{tikzpicture}
  
    \caption{\label{fig:divp3g}}
  \end{subfigure}%
  \begin{subfigure}[t]{0.33\linewidth}
    \center
    \begin{tikzpicture}[node distance =.7 cm and .7 cm]
  
        \dotnode[rt]{w}{}{below:$w$};
        \dotnode[rt]{z}{right = of w}{below:$z$};
        \dotnode[]{x}{right = of z}{above:$x$};
        \dotnode[rb]{w2}{below=of w}{below:$w_2$};
        \dotnode[rb]{z2}{right = of w2}{below:$z_2$};
        \dotnode[]{w3}{below=of w2}{below:$w_3$};
        \dotnode[rt]{z3}{right = of w3}{below:$z_3$};
        \dotnode[rb]{y}{left = of w,xshift=-1em}{above:$y$};
        \path (z) edge (x);
        
        \path (w) edge (z);
  
        \path[bidirected] (x) edge[bend right=30] (y);
        \path[bidirected] (w) edge (y);

        \path (z2) edge (x);
        \path (w2) edge (z2);
  
        \path[bidirected] (w2) edge (y);
  
        \path[bidirected,red,thick] (w) edge (z2);
        \path[bidirected,red,thick] (w2) edge (z);

        \path (z3) edge (x);
        \path[red,thick] (w3) edge (z3);
  
        \path[bidirected,red,thick] (w3) edge (y);
  
        \path[bidirected] (w2) edge (z3);
        \path[bidirected] (w3) edge (z2);
  
    \end{tikzpicture}
  
    \caption{\label{fig:divp3h}}
  \end{subfigure}%
  \begin{subfigure}[t]{0.33\linewidth}
    \center
    \begin{tikzpicture}[node distance =.7 cm and .7 cm]
  
        \dotnode[rt]{w}{}{below:$w$};
        \dotnode[rt]{z}{right = of w}{below:$z$};
        \dotnode[rb]{x}{right = of z}{above:$x$};
        \dotnode[rb]{w2}{below=of w}{below:$w_2$};
        \dotnode[rb]{z2}{right = of w2}{below:$z_2$};
        \dotnode[rtrb]{w3}{below=of w2}{below:$w_3$};
        \dotnode[rt]{z3}{right = of w3}{below:$z_3$};
        \dotnode[]{y}{left = of w,xshift=-1em}{above:$y$};
        \path (z) edge (x);
        
        \path (w) edge (z);
  
        \path[bidirected] (x) edge[bend right=30] (y);
        \path[bidirected] (w) edge (y);

        \path (z2) edge (x);
        \path (w2) edge (z2);
  
        \path[bidirected] (w2) edge (y);
  
        \path[bidirected,red,thick] (w) edge (z2);
        \path[bidirected,red,thick] (w2) edge (z);

        \path[red,thick] (z3) edge (x);
        \path (w3) edge (z3);
  
        \path[bidirected] (w3) edge (y);
  
        \path[bidirected] (w2) edge (z3);
        \path[bidirected] (w3) edge (z2);
  
    \end{tikzpicture}
  
    \caption{\label{fig:divp3i}}
  \end{subfigure}%
    \caption{Graphs used to demonstrate elements of the proof of \cref{thm:divsat}}
    \label{fig:divp3}
  \end{figure}

\divnp*
\begin{proof}
1-in-3SAT was shown to be NP-complete by \citet{schaeferComplexitySatisfiabilityProblems1978}. By \cref{thm:divsat}, we can solve 1-in-3SAT with an algorithm for tsIV. The corresponding graph is computable from
the boolean formula in polynomial time. Likewise, \cref{thm:div} describes a polynomial-time procedure for determining whether a given set can be used as a tsIV. Therefore, the problem is NP-complete.
\end{proof}

\begin{restatable}{corollary}{scivnp}
    \label{cor:scivnp}
    Given an acyclic DAG $G$ and edge $\lambda_{xy}$, finding a simple conditional instrumental set which can be used to solve for $\lambda_{xy}$ in $G$ is an NP-Complete problem.
    \end{restatable}

\begin{proof}
    1-in-3SAT was shown to be NP-complete by \citet{schaeferComplexitySatisfiabilityProblems1978}. The proof of \cref{thm:divsat}, constructed a simple conditional instrumental set for $X$ whenever there was a satisfying assignment. The corresponding graph is computable from
    the boolean formula in polynomial time. Likewise, \citep{vanderzanderSearchingGeneralizedInstrumental2016} describes a polynomial-time procedure for determining whether a given set can be used as a simple instrumental set. 
    Finally, since a tsIV does not exist whenever there is no satisfying assignment in \cref{thm:divsat}, and whenever a simple conditional instrumental set exists, a corresponding tsIV also exists (\cref{thm:scIS2div}), no simple instrumental set exists if there is no satisfying assignment.
\end{proof}

\begin{figure}
    \begin{subfigure}[t]{0.4\linewidth}
        \center
        \begin{tikzpicture}[node distance =.7 cm and .7 cm]

            \dotnode[rtrb]{w}{}{below:$w$};
            \dotnode[rt]{z}{right = of w}{below:$z$};
            \dotnode[rb]{x}{right = of z}{below:$x$};
            \dotnode[]{y}{right = of x}{below:$y$};
            \path[red] (x) edge (y);
            \path (z) edge (x);
            \path (w) edge (z);

            \path[bidirected] (x) edge[bend left=30] (y);
            \path[bidirected] (w) edge[bend left=40] (y); 
        \end{tikzpicture}

        \caption{\label{fig:divpa}}
    \end{subfigure}%
    \begin{subfigure}[t]{0.55\linewidth}
        \center
        \begin{tikzpicture}

            \node[red,very thick]  (w) [right=of a,label=above:$w$,point];
            \node[red,very thick] (wp)  [below=of w,label=below:$w'$,point];

            \node[red,very thick]  (z) [right=of w,label=above:$z$,point];
            \node (zp)  [below=of z,label=below:$z'$,point];

            \node (x) [right=of z,label=above:$x$,point];
            \node[red,very thick]  (xp)  [below=of x,label=below:$x'$,point];

            \node (yp)  [right=of xp,label=below:$y'$,point];

            \path[dashed] (w) edge (wp);
            \path[dashed] (z) edge (zp);
            \path[dashed] (x) edge (xp);

            \path (wp) edge (zp);
            \path (z) edge (w);
            \path (zp) edge (xp);
            \path (x) edge (z);
            \path (xp) edge (yp);

            \path[dashed] (x) edge (yp);

            \path[dashed] (w) edge[bend left=70] (yp);
            
        \end{tikzpicture}
        \caption{\label{fig:divpb}}
    \end{subfigure}

    \begin{subfigure}[t]{0.4\linewidth}
        \center
        \begin{tikzpicture}[node distance =.7 cm and .7 cm]

            \dotnode[rtrb]{w}{}{below:$w$};
            \dotnode[rt]{z}{right = of w}{below:$z$};
            \dotnode[]{x}{right = of z}{below:$x$};
            \dotnode[rb]{y}{right = of x}{below:$y$};
            \path (z) edge (x);
            \path (w) edge (z);

            \path[bidirected] (x) edge[bend left=30] (y);
            \path[bidirected] (w) edge[bend left=40] (y); 
        \end{tikzpicture}

        \caption{\label{fig:divpc}}
    \end{subfigure}%
    \begin{subfigure}[t]{0.55\linewidth}
        \center
        \begin{tikzpicture}

            \node[red,very thick]  (w) [right=of a,label=above:$w$,point];
            \node[red,very thick] (wp)  [below=of w,label=below:$w'$,point];

            \node[red,very thick]  (z) [right=of w,label=above:$z$,point];
            \node (zp)  [below=of z,label=below:$z'$,point];

            \node (x) [right=of z,label=above:$x$,point];
            \node  (xp)  [below=of x,label=below:$x'$,point];

            \node[red,very thick] (yp)  [right=of xp,label=below:$y'$,point];

            \path[dashed] (w) edge (wp);
            \path[dashed] (z) edge (zp);
            \path[dashed] (x) edge (xp);

            \path (wp) edge (zp);
            \path (z) edge (w);
            \path (zp) edge (xp);
            \path (x) edge (z);

            \path[dashed] (x) edge (yp);

            \path[dashed] (w) edge[bend left=70] (yp);
            
        \end{tikzpicture}
        \caption{\label{fig:divpd}}
    \end{subfigure}

    \begin{subfigure}[t]{0.4\linewidth}
        \center
        \begin{tikzpicture}[node distance =.7 cm and .7 cm]

            \dotnode[rt]{w}{}{below:$w$};
            \dotnode[rtrb]{z}{right = of w}{below:$z$};
            \dotnode[rb]{x}{right = of z}{below:$x$};
            \dotnode[]{y}{right = of x}{below:$y$};
            \path[red] (x) edge (y);
            \path (z) edge (x);
            \path (w) edge (z);

            \path[bidirected] (x) edge[bend left=30] (y);
            \path[bidirected] (w) edge[bend left=40] (y);

            \dotnode[]{w2}{below=of w}{below:$w_2$};
            \dotnode[]{z2}{right = of w2}{below:$z_2$};
            \path[red] (x) edge (y);
            \path (z2) edge (x);
            \path (w2) edge (z2);

            \path[bidirected] (w2) edge[bend right=60] (y);

            \path[bidirected] (w) edge (z2);
            \path[bidirected] (w2) edge (z);
 
        \end{tikzpicture}

        \caption{\label{fig:divpe}}
    \end{subfigure}%
    \begin{subfigure}[t]{0.55\linewidth}
        \center
        \begin{tikzpicture}

            \node[red,very thick]  (w) [label=above:$w$,point];
            \node (wp)  [below=of w,label=below:$w'$,point];

            \node[red,very thick]  (z) [right=of w,label=above:$z$,point];
            \node[red,very thick] (zp)  [below=of z,label=below:$z'$,point];

            \node (x) [right=of z,label=above:$x$,point];
            \node[red,very thick]  (xp)  [below=of x,label=below:$x'$,point];

            \node (yp)  [right=of xp,label=below:$y'$,point];

            \path[dashed,thick,red] (w) edge (wp);
            \path[dashed] (z) edge (zp);
            \path[dashed] (x) edge (xp);

            \path[thick,red] (wp) edge (zp);
            \path (z) edge (w);
            \path (zp) edge (xp);
            \path (x) edge (z);
            \path (xp) edge (yp);

            \path[dashed] (x) edge (yp);

            \path[dashed] (w) edge[bend left=70] (yp);

            \node  (w2) [below=of wp,label=above:$w_2$,point];
            \node (wp2)  [below=of w2,label=below:$w_2'$,point];

            \node  (z2) [right=of w2,label=above:$z_2$,point];
            \node (zp2)  [below=of z2,label=below:$z_2'$,point];

            \path[dashed] (w2) edge (wp2);
            \path[dashed] (z2) edge (zp2);

            \path[thick,red] (wp2) edge (zp2);
            \path (z2) edge (w2);
            \path[thick,red] (zp2) edge (xp);
            \path (x) edge (z2);

            \path[dashed] (w2) edge[bend right=40] (yp);

            \path[dashed] (z2) edge (wp);
            \path[dashed,thick,red] (z) edge (wp2);

            \path[dashed] (w) edge (zp2);
            \path[dashed] (w2) edge (zp);
            
        \end{tikzpicture}
        \caption{\label{fig:divpf}}
    \end{subfigure}

    \begin{subfigure}[t]{0.4\linewidth}
        \center
        \begin{tikzpicture}[node distance =.7 cm and .7 cm]

            \dotnode[rtrb]{w}{}{below:$w$};
            \dotnode[rt]{z}{right = of w}{below:$z$};
            \dotnode[]{x}{right = of z}{below:$x$};
            \dotnode[rb]{y}{right = of x}{below:$y$};
            \path[red] (x) edge (y);
            \path (z) edge (x);
            
            \path (w) edge (z);

            \path[bidirected] (x) edge[bend left=30] (y);
            \path[bidirected] (w) edge[bend left=40] (y);

            \dotnode[rtrb]{w2}{below=of w}{below:$w_2$};
            \dotnode[rt]{z2}{right = of w2}{below:$z_2$};
            \dotnode[rb]{x2}{right = of z2}{below:$x_2$};
            \path (z) edge (x2);
            \path (x2) edge (y);
            \path (z2) edge (x);
            \path (z2) edge (x2);
            \path[bidirected] (x2) edge[bend left=30] (y);
            \path (w2) edge (z2);

            \path[bidirected] (w2) edge[bend right=80] (y);

        \end{tikzpicture}

        \caption{\label{fig:divpg}}
    \end{subfigure}%
    \begin{subfigure}[t]{0.55\linewidth}
        \center
        \begin{tikzpicture}

            \node[red,very thick]  (w) [right=of a,label=above:$w$,point];
            \node[red,very thick] (wp)  [below=of w,label=below:$w'$,point];

            \node[red,very thick]  (z) [right=of w,label=above:$z$,point];
            \node (zp)  [below=of z,label=below:$z'$,point];

            \node (x) [right=of z,label=above:$x$,point];
            \node  (xp)  [below=of x,label=below:$x'$,point];

            \node[red,very thick] (yp)  [right=of xp,label=right:$y'$,point];

            \path[dashed] (w) edge (wp);
            \path[dashed] (z) edge (zp);
            \path[dashed] (x) edge (xp);

            \path (wp) edge (zp);
            \path (z) edge (w);
            \path (zp) edge (xp);
            \path (x) edge (z);
            \path (xp) edge (yp);

            \path[dashed] (x) edge (yp);

            \path[dashed] (w) edge[bend left=70] (yp);

            \node[red,very thick]  (w2) [below=of wp,label=above:$w_2$,point];
            \node[red,very thick] (wp2)  [below=of w2,label=below:$w_2'$,point];

            \node[red,very thick]  (z2) [right=of w2,label=above:$z_2$,point];
            \node (zp2)  [below=of z2,label=below:$z_2'$,point];

            \node (x2) [right=of z2,label=above:$x_2$,point];
            \node[red,very thick]  (xp2)  [below=of x2,label=below:$x'_2$,point];

            \path[dashed] (w2) edge (wp2);
            \path[dashed] (z2) edge (zp2);

            \path (wp2) edge (zp2);
            \path (z2) edge (w2);
            \path (zp2) edge (xp);
            \path (x) edge (z2);

            \path[dashed] (w2) edge[bend right=80] (yp);

            \path[dashed] (x2) edge (yp);
            \path[dashed] (x2) edge (xp2);
            \path (xp2) edge (yp);

            \path (zp) edge (xp2);
            \path (x2) edge (z);

            \path (zp2) edge (xp2);
            \path (x2) edge (z2);

        \end{tikzpicture}
        \caption{\label{fig:divph}}
    \end{subfigure}
    
    \caption{Graphs used to illustrate elements of the proof of \cref{thm:divsat}}
    \label{fig:divp}
\end{figure}
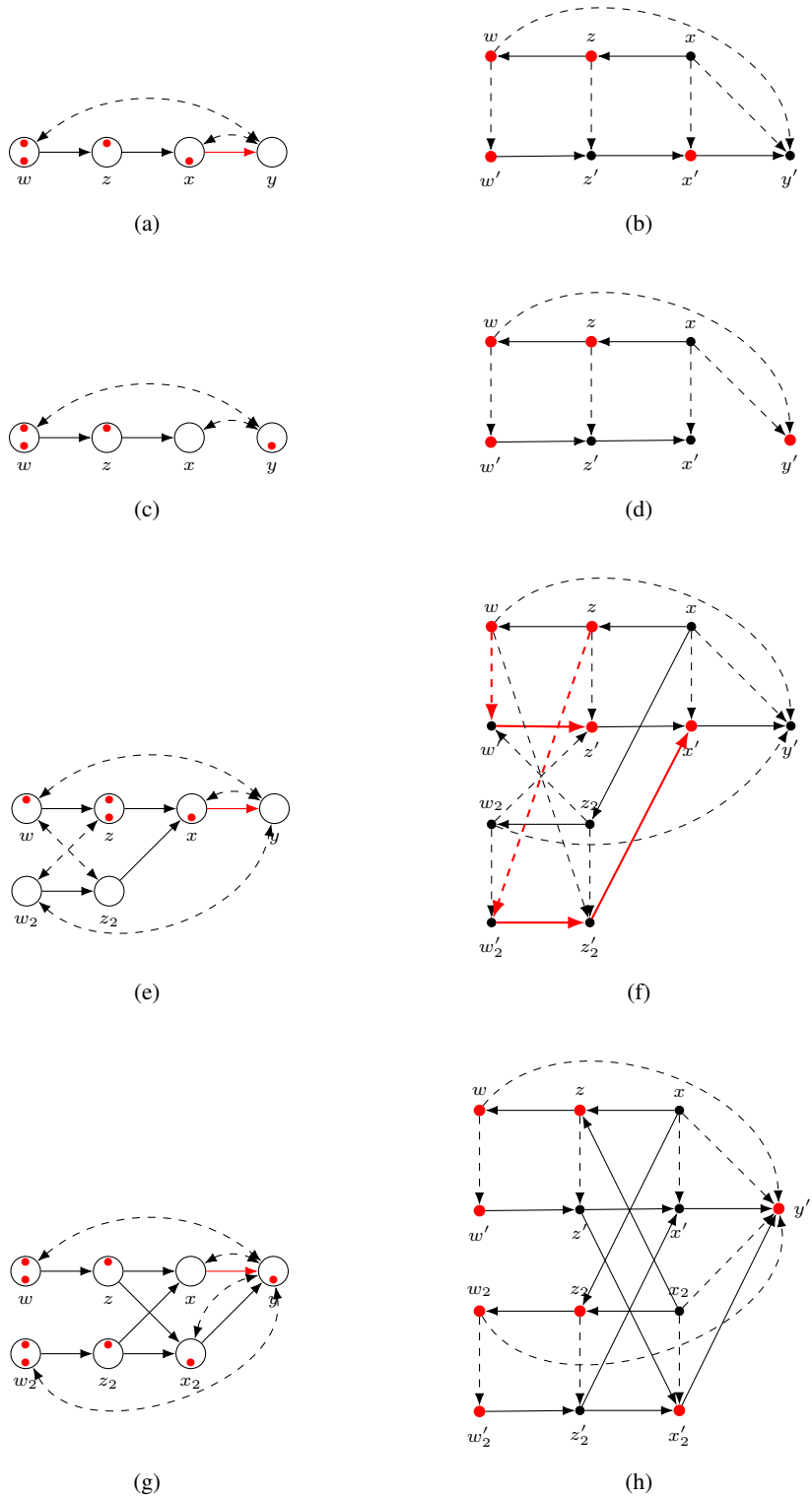

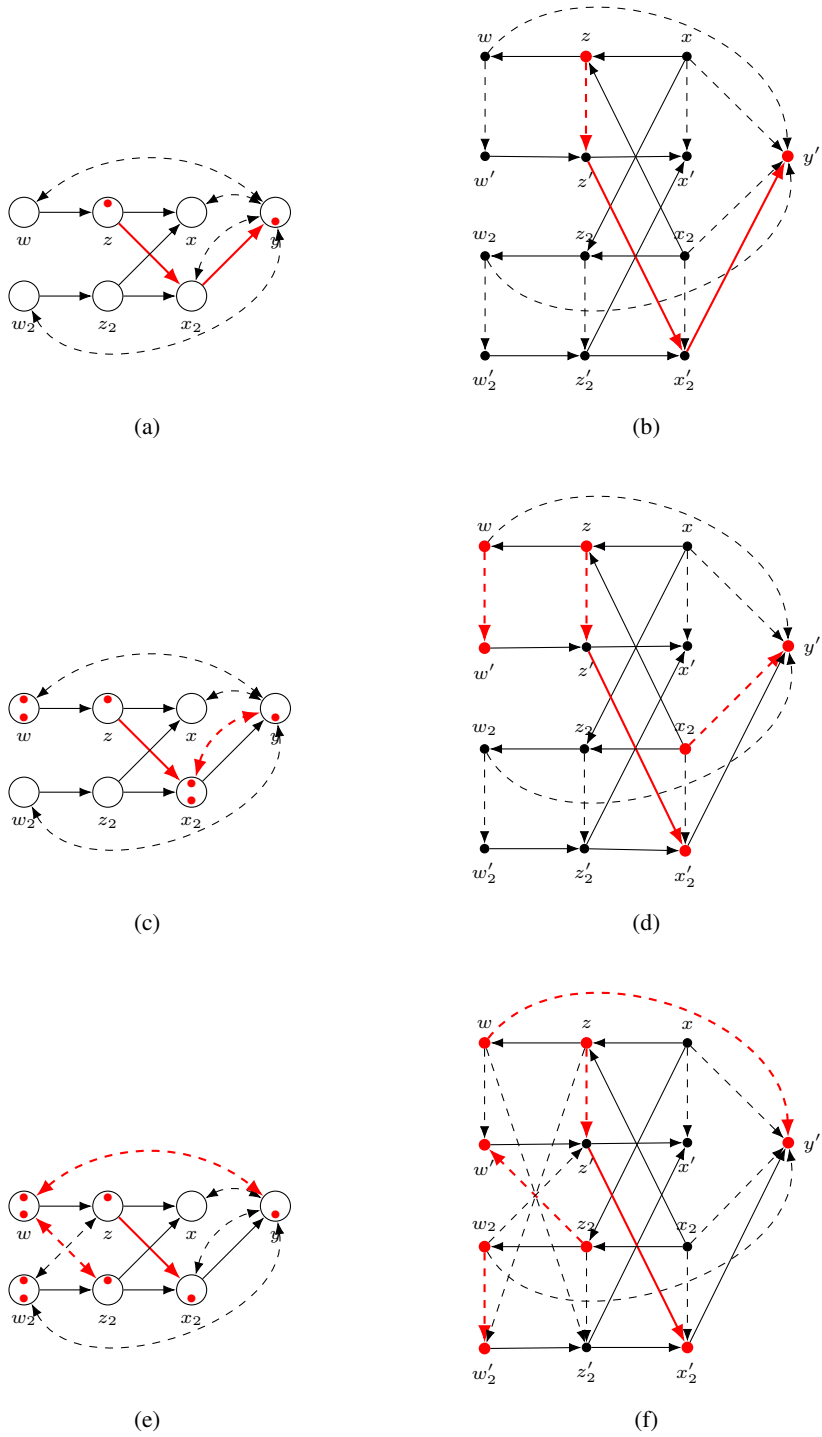
\begin{figure}

    \begin{subfigure}[t]{0.4\linewidth}
        \center
        \begin{tikzpicture}[node distance =.7 cm and .7 cm]

            \dotnode[]{w}{}{below:$w$};
            \dotnode[rt]{z}{right = of w}{below:$z$};
            \dotnode[]{x}{right = of z}{below:$x$};
            \dotnode[rb]{y}{right = of x}{below:$y$};
            \path (z) edge (x);
            
            \path (w) edge (z);

            \path[bidirected] (x) edge[bend left=30] (y);
            \path[bidirected] (w) edge[bend left=40] (y);

            \dotnode[]{w2}{below=of w}{below:$w_2$};
            \dotnode[]{z2}{right = of w2}{below:$z_2$};
            \dotnode[]{x2}{right = of z2}{below:$x_2$};
            \path[red,thick] (z) edge (x2);
            \path[red,thick] (x2) edge (y);
            \path (z2) edge (x);
            \path (z2) edge (x2);
            \path[bidirected] (x2) edge[bend left=30] (y);
            \path (w2) edge (z2);

            \path[bidirected] (w2) edge[bend right=80] (y);

        \end{tikzpicture}

        \caption{\label{fig:divp2a}}
    \end{subfigure}%
    \begin{subfigure}[t]{0.55\linewidth}
        \center
        \begin{tikzpicture}

            \node  (w) [right=of a,label=above:$w$,point];
            \node (wp)  [below=of w,label=below:$w'$,point];

            \node[red,very thick]  (z) [right=of w,label=above:$z$,point];
            \node (zp)  [below=of z,label=below:$z'$,point];

            \node (x) [right=of z,label=above:$x$,point];
            \node  (xp)  [below=of x,label=below:$x'$,point];

            \node[red,very thick] (yp)  [right=of xp,label=right:$y'$,point];

            \path[dashed] (w) edge (wp);
            \path[dashed,red,thick] (z) edge (zp);
            \path[dashed] (x) edge (xp);

            \path (wp) edge (zp);
            \path (z) edge (w);
            \path (zp) edge (xp);
            \path (x) edge (z);

            \path[dashed] (x) edge (yp);

            \path[dashed] (w) edge[bend left=70] (yp);

            \node  (w2) [below=of wp,label=above:$w_2$,point];
            \node (wp2)  [below=of w2,label=below:$w_2'$,point];

            \node (z2) [right=of w2,label=above:$z_2$,point];
            \node (zp2)  [below=of z2,label=below:$z_2'$,point];

            \node (x2) [right=of z2,label=above:$x_2$,point];
            \node (xp2)  [below=of x2,label=below:$x'_2$,point];

            \path[dashed] (w2) edge (wp2);
            \path[dashed] (z2) edge (zp2);

            \path (wp2) edge (zp2);
            \path (z2) edge (w2);
            \path (zp2) edge (xp);
            \path (x) edge (z2);

            \path[dashed] (w2) edge[bend right=80] (yp);

            \path[dashed] (x2) edge (yp);
            \path[dashed] (x2) edge (xp2);
            \path[red,thick] (xp2) edge (yp);

            \path[red,thick] (zp) edge (xp2);
            \path (x2) edge (z);

            \path (zp2) edge (xp2);
            \path (x2) edge (z2);

        \end{tikzpicture}
        \caption{\label{fig:divp2b}}
    \end{subfigure}

    \begin{subfigure}[t]{0.4\linewidth}
        \center
        \begin{tikzpicture}[node distance =.7 cm and .7 cm]

            \dotnode[rtrb]{w}{}{below:$w$};
            \dotnode[rt]{z}{right = of w}{below:$z$};
            \dotnode[]{x}{right = of z}{below:$x$};
            \dotnode[rb]{y}{right = of x}{below:$y$};
            \path (z) edge (x);
            
            \path (w) edge (z);

            \path[bidirected] (x) edge[bend left=30] (y);
            \path[bidirected] (w) edge[bend left=40] (y);

            \dotnode[]{w2}{below=of w}{below:$w_2$};
            \dotnode[]{z2}{right = of w2}{below:$z_2$};
            \dotnode[rtrb]{x2}{right = of z2}{below:$x_2$};
            \path[red,thick] (z) edge (x2);
            \path (x2) edge (y);
            \path (z2) edge (x);
            \path (z2) edge (x2);
            \path[bidirected,red,thick] (x2) edge[bend left=30] (y);
            \path (w2) edge (z2);

            \path[bidirected] (w2) edge[bend right=80] (y);

        \end{tikzpicture}

        \caption{\label{fig:divp2c}}
    \end{subfigure}%
    \begin{subfigure}[t]{0.55\linewidth}
        \center
        \begin{tikzpicture}

            \node[red,very thick]  (w) [right=of a,label=above:$w$,point];
            \node[red,very thick] (wp)  [below=of w,label=below:$w'$,point];

            \node[red,very thick]  (z) [right=of w,label=above:$z$,point];
            \node (zp)  [below=of z,label=below:$z'$,point];

            \node (x) [right=of z,label=above:$x$,point];
            \node  (xp)  [below=of x,label=below:$x'$,point];

            \node[red,very thick] (yp)  [right=of xp,label=right:$y'$,point];

            \path[dashed,red,thick] (w) edge (wp);
            \path[dashed,red,thick] (z) edge (zp);
            \path[dashed] (x) edge (xp);

            \path (wp) edge (zp);
            \path (z) edge (w);
            \path (zp) edge (xp);
            \path (x) edge (z);

            \path[dashed] (x) edge (yp);

            \path[dashed] (w) edge[bend left=70] (yp);

            \node  (w2) [below=of wp,label=above:$w_2$,point];
            \node (wp2)  [below=of w2,label=below:$w_2'$,point];

            \node (z2) [right=of w2,label=above:$z_2$,point];
            \node (zp2)  [below=of z2,label=below:$z_2'$,point];

            \node[red,very thick] (x2) [right=of z2,label=above:$x_2$,point];
            \node[red,very thick] (xp2)  [below=of x2,label=below:$x'_2$,point];

            \path[dashed] (w2) edge (wp2);
            \path[dashed] (z2) edge (zp2);

            \path (wp2) edge (zp2);
            \path (z2) edge (w2);
            \path (zp2) edge (xp);
            \path (x) edge (z2);

            \path[dashed] (w2) edge[bend right=80] (yp);

            \path[dashed,red,thick] (x2) edge (yp);
            \path[dashed] (x2) edge (xp2);
            \path (xp2) edge (yp);

            \path[red,thick] (zp) edge (xp2);
            \path (x2) edge (z);

            \path (zp2) edge (xp2);
            \path (x2) edge (z2);

        \end{tikzpicture}
        \caption{\label{fig:divp2d}}
    \end{subfigure}

    \begin{subfigure}[t]{0.4\linewidth}
        \center
        \begin{tikzpicture}[node distance =.7 cm and .7 cm]

            \dotnode[rtrb]{w}{}{below:$w$};
            \dotnode[rt]{z}{right = of w}{below:$z$};
            \dotnode[]{x}{right = of z}{below:$x$};
            \dotnode[rb]{y}{right = of x}{below:$y$};
            \path (z) edge (x);
            
            \path (w) edge (z);

            \path[bidirected] (x) edge[bend left=30] (y);
            \path[bidirected,red,thick] (w) edge[bend left=40] (y);

            \dotnode[rtrb]{w2}{below=of w}{below:$w_2$};
            \dotnode[rt]{z2}{right = of w2}{below:$z_2$};
            \dotnode[rb]{x2}{right = of z2}{below:$x_2$};
            \path[red,thick] (z) edge (x2);
            \path (x2) edge (y);
            \path (z2) edge (x);
            \path (z2) edge (x2);
            \path[bidirected] (x2) edge[bend left=30] (y);
            \path (w2) edge (z2);

            \path[bidirected] (w2) edge[bend right=80] (y);

            \path[bidirected,red,thick] (w) edge (z2);
            \path[bidirected] (w2) edge (z);

        \end{tikzpicture}

        \caption{\label{fig:divp2e}}
    \end{subfigure}%
    \begin{subfigure}[t]{0.55\linewidth}
        \center
        \begin{tikzpicture}

            \node[red,very thick]  (w) [right=of a,label=above:$w$,point];
            \node[red,very thick] (wp)  [below=of w,label=below:$w'$,point];

            \node[red,very thick]  (z) [right=of w,label=above:$z$,point];
            \node (zp)  [below=of z,label=below:$z'$,point];

            \node (x) [right=of z,label=above:$x$,point];
            \node  (xp)  [below=of x,label=below:$x'$,point];

            \node[red,very thick] (yp)  [right=of xp,label=right:$y'$,point];

            \path[dashed] (w) edge (wp);
            \path[dashed,red,thick] (z) edge (zp);
            \path[dashed] (x) edge (xp);

            \path (wp) edge (zp);
            \path (z) edge (w);
            \path (zp) edge (xp);
            \path (x) edge (z);

            \path[dashed] (x) edge (yp);

            \path[dashed,red,thick] (w) edge[bend left=70] (yp);

            \node[red,very thick]  (w2) [below=of wp,label=above:$w_2$,point];
            \node[red,very thick] (wp2)  [below=of w2,label=below:$w_2'$,point];

            \node[red,very thick]  (z2) [right=of w2,label=above:$z_2$,point];
            \node (zp2)  [below=of z2,label=below:$z_2'$,point];

            \node (x2) [right=of z2,label=above:$x_2$,point];
            \node[red,very thick]  (xp2)  [below=of x2,label=below:$x'_2$,point];

            \path[dashed,red,thick] (w2) edge (wp2);
            \path[dashed] (z2) edge (zp2);

            \path (wp2) edge (zp2);
            \path (z2) edge (w2);
            \path (zp2) edge (xp);
            \path (x) edge (z2);

            \path[dashed] (w2) edge[bend right=80] (yp);

            \path[dashed] (x2) edge (yp);
            \path[dashed] (x2) edge (xp2);
            \path (xp2) edge (yp);

            \path[red,thick] (zp) edge (xp2);
            \path (x2) edge (z);

            \path (zp2) edge (xp2);
            \path (x2) edge (z2);

            \path[dashed] (w) edge (zp2);
            \path[dashed,red,thick] (z2) edge (wp);
            \path[dashed] (z) edge (wp2);
            \path[dashed] (w2) edge (zp);

        \end{tikzpicture}
        \caption{\label{fig:divp2f}}
    \end{subfigure}

    \caption{Graphs used to illustrate elements of the proof of \cref{thm:divsat}}
    \label{fig:divp2}
\end{figure}

\newpage
\subsection{Instrumental Cutsets}

\unconditionedcutset*

\begin{proof}
Our proof mirrors the work of \cite{weihsDeterminantalGeneralizationsInstrumental2017}, except with the auxiliary flow graph instead of $G_{flow}$.
We construct the auxiliary flow graph $G_{aux}$. Condition 1 guarantees that $G_{aux}$ has a set of nonitersecting paths from $S^*$ to $T'\cup\{x'\}$, which means that $\Sigma_{s^*,T\cup\{x\}}$ is full rank (\cref{lemma:auxgvl}).
Then, condition 2 and 3 guarantees that there is no set of nonintersecting paths from $S^*$ to $T'\cup\{y'\}$ that does not go through the edge $x'$. Each set of paths from $S^*$ to $T'\cup\{x'\}$ can be extended
to a path from $S^*$ to $T'\cup\{y'\}$ by adding $x'\rightarrow y'$ to the path that ends at $x'$. Using \cref{lemma:auxgvl} we have 
$$
\det \Sigma_{S^*, T\cup\{y\}} = \lambda_{xy}\det \Sigma_{S^*, T\cup\{x\}}
    $$
which gives an equation for $\lambda_{xy}$ after division.

If edge from $w_i$ incoming to $y$ is known, then all treks to $y$ through $\lambda_{w_iy}$ can be observed with $\lambda_{w_iy}\Sigma_{S^*, T\cup\{x\}}$ (once again, due directly to the interpretation of determinants of minors in the covariance matrix as sets of nonintersecting paths due to \cref{lemma:auxgvl}), giving all paths not passing any of the known edges with:

$$\det \Sigma_{S^*, T\cup\{y\}}- \sum_{\lambda_{w_iy}\in \Lambda^*} \det\Sigma_{S^*,T\cup \{w_i\}}$$
This leads to the full statement of the formula, mirroring the original equation of \cref{thm:div}.

\end{proof}

\subsumesAVS*

\begin{proof}
First, neither gHTC nor AVS can identify $\lambda_{x_1y}$ in \cref{fig:cutsetiv}, but it is identifiable with IC.

Let $S$ be a set of auxiliary variables that satisfies the requirements of an AVS to a set of edges $X$ where $X\subseteq Pa(y)$.

Then there is a full flow between $S^*$ and $X$ (there are paths with no sided intersection), and since all $Pa(y)$ reachable from $S^*$ are in $X$, and $|X|=|S|$,  they correspond to a match-block, meaning that the minimal cutset will also have a matchblock, so by \cref{thm:cutsetsubset}, it satisfies the rules of \cref{thm:unconditionedcutset}.

Finally, AVS subsumes gHTC, so ICs subsume both.
\end{proof}

\begin{restatable}{theorem}{auxiliarydiv}
\label{thm:auxiliarydiv}
    \textbf{(Auxiliary tsIV)} 
    Let $G=(V,D,B)$ be a mixed graph, $w_0\rightarrow v \in G$, and suppose that the edges 
    $w_1\rightarrow v,...,w_l\rightarrow v \in G$ are known to be generically (rationally) identifiable.
Let $G_{aux}'$ be $G_{aux}$ with the edges $w_0'\rightarrow v',...,w_l'\rightarrow v'$ removed. 
Suppose there are sets $S\subset V$ and $T\subset V \setminus \{v,w_0\}$ such that $|S|=|T|+1=k$ and
\begin{enumerate}
    \item $De(v)\cap (T\cup \{v\}) = \emptyset$,
    \item the max-flow from $S$ to $T'\cup \{w'_0\}$ in $G'_{aux}$ equals $k$, and
    \item the max-flow from $S$ to $T'\cup \{v'\}$ in $G'_{aux}$ is $<k$,
\end{enumerate}
then $w_0\rightarrow v$ is rationally identifiable by the equation
$$
\lambda_{w_0v} = \frac{\left|\Sigma_{S,T\cup \{v\}}\right| - \sum_{i=1}^l \left|\Sigma_{S,T\cup \{w_i\}}\right|}{\left|\Sigma_{S,T\cup \{w_0\}}\right|}
$$
\end{restatable}
\begin{proof}
    Identical to \cref{thm:unconditionedcutset}.
\end{proof}

\subsubsection{A Discussion of Conditonal IC Complexity}

\label{sec:cacnphard}

Many models previously only identifiable with conditional cAVs are identifiable with IC, 
such as the one in \cref{fig:cavtosubset}. The key here is that in many situations, the back paths can
be removed through edges identifiable in previous iterations of the algorithm, and there is no need to block
paths to parents of $y$, since the IC can simply use all parents of $y$ in the set.

This shows why the proof in \cref{thm:divsat} cannot be directly applied to IC. The structure shown in \cref{fig:basediv_clause} has its edge $\lambda_{az_1}$ and $\lambda_{bz_2}$ identifiable directly, and therefore
there is no need for conditioning at all - the edge is identifiable with standard IC.

However, we can use the structure in \cref{fig:cAVonly}, which cannot have its back paths removed to replace the original structure of a literal in the proof. The construction would be identical to the one given in \cref{thm:divsat}, but with 2 $x$ variables per clause, and each literal consisting of the structure in \cref{fig:cAVonly}. That way, we can show that both the auxiliary tsIV and a conditional version of IC would also be NP-hard to find.

\subsubsection{Algorithm for Instrumental Cutsets}

\begin{figure}
\center
    \begin{subfigure}[t]{0.27\textwidth}
        \center
    \begin{tikzpicture}[node distance=0.6cm]
        \dotnode[]{y}{}{below:y};
        \dotnode[b]{x2}{above =of y}{right:$x_2$};
        \dotnode[]{x3}{right=of x2}{right:$x_3$};
        \dotnode[b]{x1}{left=of x2}{right:$x_1$};   
        
        \dotnode[]{w1}{above=of x1,yshift=2em}{left:$w_1$};
        \dotnode[t]{w2}{above=of x3,yshift=-1em}{above:$w_2$};
        \dotnode[t]{z}{above=of x2}{above:$z$};
        \path (x1) edge (y);
        \path (x2) edge (y);
        \path (x3) edge (y);
    
        \path[bidirected] (x1) edge[bend right=30] (y);
        \path[bidirected] (x2) edge[bend right=10] (y);
        \path[bidirected] (x3) edge[bend left=30] (y);

        \path[bidirected] (w1) edge[bend right=60] (y);

        \path[red] (w1) edge (z);
        \path[red] (z) edge (w2);
        \path (w2) edge (x3);
        \path (w2) edge (x2);
        \path (z) edge (x1);
    
    \end{tikzpicture}
    \caption{\label{fig:cavtosubset}}
\end{subfigure}
\begin{subfigure}[t]{0.25\textwidth}
    \center
    \begin{tikzpicture}[node distance=0.6cm]
        \dotnode[]{y}{}{below:y};
        \dotnode[b]{x1}{above left=of y}{left:$x_1$};
        \dotnode[b]{x2}{above right=of y}{left:$x_2$};

        \dotnode[]{z1}{above=of x1,yshift=3.5em}{left:$z_1$};
        \dotnode[t]{z2}{above=of x2}{right:$z_2$};
        \dotnode[]{z3}{above=of z2}{above:$z_3$};
        \dotnode[t]{z4}{left=of z2}{right:$z_4$};

        \path[red] (x2) edge (y);
        \path (x1) edge (y);
        \path (z1) edge (x1);
        \path (z2) edge (x2);
        \path (z3) edge (z2);
        \path (z3) edge (z4);
        \path (z1) edge (z3);

        \path[bidirected] (x1) edge[bend right=30] (y);
        \path[bidirected] (x2) edge[bend left=30] (y);

        \path[bidirected] (z3) edge[bend left=90] (y);

    \end{tikzpicture}
    \caption{\label{fig:htrhard}}
\end{subfigure}
\begin{subfigure}[t]{0.2\textwidth}
    \center
    \begin{tikzpicture}[node distance=0.6cm]
        \dotnode[]{y}{}{below:y};
        \dotnode[]{x1}{above left=of y}{left:$x_1$};
        \dotnode[]{x2}{above right=of y}{above:$x_2$};

        \dotnode[]{z1}{above=of x1}{left:$z_1$};

        \path[red] (x2) edge (y);
        \path (x1) edge (y);
        \path (z1) edge (x1);
        \path (x1) edge (x2);

        \path[bidirected] (x1) edge[bend right=30] (y);
        \path[bidirected] (x2) edge[bend left=30] (y);

    \end{tikzpicture}
    \caption{\label{fig:cutmatchblock}}
\end{subfigure}
\begin{subfigure}[t]{0.23\linewidth}
    \center
    \begin{tikzpicture}[node distance =0.5 cm and 0.5 cm]
        \dotnode[]{y}{}{below:y};
        \dotnode[b]{x2}{above=of y}{left:$x_2$};
        
        \dotnode[]{x3}{right=of x2}{left:$x_3$};
        \dotnode[]{w}{above=of x2}{right:$w$};
        \dotnode[b]{x1}{left=of w}{left:$x_1$};

        \dotnode[t]{z1}{above left=of w}{left:$z_1$};
        \dotnode[t]{z2}{above right=of w}{left:$z_2$};

        \path (x1) edge (w);
        \path (z1) edge (x1);
        \path (z2) edge (w);
        \path (w) edge (x2);
        \path (w) edge (x3);
        \path[red] (x1) edge[bend right=20] (y);
        \path (x2) edge (y);
        \path (x3) edge (y);
        \path[bidirected] (x1) edge[bend right=40] (y);
        \path[bidirected] (x2) edge[bend left=20] (y);
        \path[bidirected] (w) edge[bend left=25] (y);

        \path[bidirected] (x3) edge[bend left=20] (y);
    \end{tikzpicture}
    \caption{\label{fig:cutsetparents}}
\end{subfigure}%
    
    \caption{In (a), $z$ a conditional IV for $\lambda_{x_1y}$ given $w_1$ and $w_2$. 
    However, we can use $w_1$ to solve for $\lambda_{w_1z}$ and $\lambda_{zw_2}$, at which point we can identify $\lambda_{x_1y}$ using \cref{thm:unconditionedcutset}.
    In (b), we show the example of an IC for $\lambda_{x_2y}$ where $S=\{z_2,z_4\}$ and $T=\{x_1\}$. Critically, the back-path through the bidirected edge to $y$ from $z_2$ is blocked
    by the path from $z_4$ to $x_1$. The na\"ive construction of a half-trek IC would have $S'=\{z_1,z_2\}$, which would consequently allow a path from $z_2 \leftarrow z_3 \leftrightarrow y$. 
    In (c) is demonstrated the reason a match-block is required even when the closest cutset is limited to parents of $y$. $z_1$ has closest cutset at $x'_1$, but does not have a match-block,
    since $x'_1$ also has a path to $x'_2$, which does not have a corresponding match. 
    Finally, (d) demonstrates why incoming edges to the cut $C$ are removed, with the example where an IC exists for $\lambda_{x_1y}$ but $x_2$ and $x_3$ both being descendants of $x_1$ if edges incident to $w$ are not removed.
    }
\end{figure}
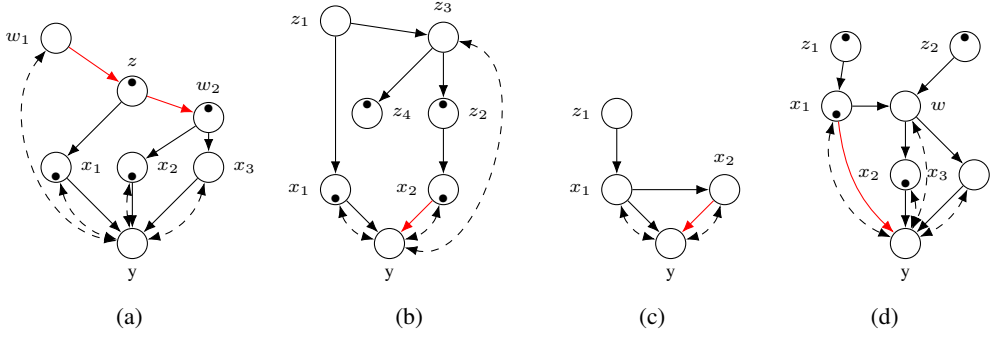
    
Before developing an algorithm for IC, we first show some helper lemmas which allow us to gain an intuition about the problem.

\begin{lemma}
    \label{lemma:cutmatchblock}
Given a DAG $G=(V,D)$, a set of sources $S$, and sinks $T$, and a vertex min-cut $C$ between $S$ and $T$ closest to $T$, 
then if there exists a match-block between $C_m\subseteq C$ and $T_m\subseteq T$, then $C_m = T_m$.
\end{lemma}

\begin{proof}
Suppose that $\exists c_i \in C_m\setminus T$. We can then replace $C_m$ in $C$ with $T_m$, to create a closer min-cut of the same size as $C$ (a contradiction), 
because $|C_m|=|T_m|$ and $T_m$ blocks all paths outgoing from $C_m$ to $T$. Furthermore, if $\exists c_i \in C_m \setminus T_m$ where $c_i\in T$, 
then it is not part of a match-block, since it has an unblocked path to $c_i \in T$ - another contradiction. 
\end{proof}

This means that a match-block between $C$ and $T$ can be described with a single set $C_m \subseteq T$, which describes both start and endpoints of the matchblock.
Note, however, that we still need to find a full match-block, as shown in the example in \cref{fig:cutmatchblock}, where even when $T=Pa(y)$, no match-block exists, because $x_1$ has a path to $x_2$.

\begin{lemma}
    \label{lemma:flowdecomposition}
Given a DAG $G$, a set of sources $S$, sinks $T$, and a vertex min-cut $C$ between $S$ and $T$, there exist max flow $\mathcal{F}_C$ from $S$ to $T$ which can be decomposed 
into max-flows $\mathcal{F}_S$ from $S$ to $C$ 
and $\mathcal{F}_T$ from $C$ to $T$, with $|\mathcal{F}_S|=|\mathcal{F}_T|=|\mathcal{F}_C|=|C|$. 
\end{lemma}

\begin{proof}
This is by defintion of a vertex min-cut.
\end{proof}

\begin{lemma}
\label{lemma:nonblockflow}
Given a DAG $G$, a set of sources $S$, sinks $T$, if there is a vertex max-flow $\mathcal{F}$ of size $k$ between $S$ and $T$, with nonzero-flow start and endpoints $S_f\subseteq S$, $T_f\subseteq T$ respectively, and $S_m\subseteq S_f,T_m\subseteq T$ consitute a match-block,
then $\mathcal{F}$ can be decomposed into $\mathcal{F}^-$ between $S^-=S_f\setminus  S_m$ and $T^- = T_f\setminus T_m$ of size $k-|S_m|$, and a full flow $\mathcal{F}_m$ between $S_m$ and $T_m$.
\end{lemma}

\begin{proof}
We know that $\mathcal{F}_m$ is a flow of size $|S_m|=|T_m|$. By the definition of match-block, we know that all descendants of $S_m$ that are in $T$ are also in $T_m$. This means that any path starting from $S_f$ that crosses
the descendants of $S_m$ must end on $T_m$. Suppose that $\exists s_i \in S^-$ that is matched with $t_i \in T_m$. Consequently, one of the $S_m$ cannot be matched, a contradiction, since $S_m$ \textit{all} have nonzero flow in $\mathcal{F}$.
We can therefore conclude that $\mathcal{F}$ can be decomposed into a full flow from $S_m$ to $T_m$, contained entirely in the descendants of $S_m$ ($\mathcal{F}_m$), 
and a flow $\mathcal{F}^-$ which does not cross the descendants of $S_m$, from $S^-$ to $T^-$. Since the flows can be decomposed into a flow through the match-block, and a flow avoiding the match-block, we know that $|\mathcal{F}| = |\mathcal{F}_m| + |\mathcal{F}^-|$, which completes the proof (remember $|\mathcal{F}_m|=|S_m|$ and $|\mathcal{F}|=k$).

\end{proof}

\begin{lemma}
\label{lemma:blockcut}
Given a DAG $G$, a set of source nodes $S$, sinks $T$, and a node $x\in An(T), x\notin De(S)$, if there does not exist a max-flow $\mathcal{F}$ from $S$ to $T$, 
and a path $p$ from $x$ to any $t_i\in T$ such that $p$ doesn't cross any flow from $\mathcal{F}$,
then there exists a closest min-cut $C$ between $S$ and $T$ where at least one element $c_i\in C$, such that $c_i\notin S$, or all paths from $x$ cross $S$.
\end{lemma}

\begin{proof}
Suppose that the max-flow $\mathcal{F}$ is of size $k$. We now take a new max-flow $\mathcal{F}'$ from $S\cup\{x\}$ to $T$. If this flow were of size $k+1$, we could decompose it into a flow of size $k$ between $S$ and $T$, and a flow of size 1
between $x$ and $T$ (since we already know the max-flow between $S$ and $T$ is $k$). We can construct a path $p$ that contradicts the theorem's statement using the flow from $x$ to $T$. This means that $\mathcal{F}'$ must be of size $k$.

We therefore have $|S\cup \{x\}| > k$ and $|T|>k$, meaning that there is a bottleneck of size $k$ between the two sets (ie, a set of nodes forming a min-cut). All paths from $x$ to $T$ must cross elements of a closest min-cut $C$.

If there exists a path from $x$ to $T$ not crossing $S$, the closest-min-cut of the path must be in $De(S)$, since otherwise the full min-cut would be $k+1$, giving a flow of $k+1$, a contradiction.

We therefore know that there exists a cut $C$ of size $k$ contained in $De(S)$, which cuts $S\cup\{x\}$ from $T$. The same cut therefore must cut $S$ from $T$ - and if $x$ has a path to $T$ not crossing $S$, an element of $C$ must be along the path, meaning
that $c_i\in C$, but $c_i\notin S$.
\end{proof}

\cutsetsubset*

\begin{proof}

$\Rightarrow$: Suppose that a match-block $C_m=T_m$ exists between $C$ and $Pa(y)$ (see \cref{lemma:cutmatchblock}) when incident edges to $C$ are removed. We show that we can construct sets $S_f,T_f$ satisfying conditions 1 and 2 for all $x\in T_m$.
Let $\mathcal{F}$ be a (vertex) max-flow from $S$ to $Pa(y)$ in the graph with all edges present. Next, let $S_f\subseteq S,T_f\subseteq Pa(y)\setminus\{x\}$ be the nodes of $S$ and $T$ respectively with nonzero flow in $\mathcal{F}$. We define $k=|S_f|$, automatically satisfying condition 1.
We can decompose $\mathcal{F}$ into a max-flow of size $k=|C|$ from $S_f$ to $C$ ($\mathcal{F}_S$) and from $C$ to $T_f\cup \{x\}$ ($\mathcal{F}_T$) using \cref{lemma:flowdecomposition}. 
Finally, we define $C^-_m = C\setminus C_m$ and $T^-_m=T_f\setminus C_m$. 

With $x\rightarrow y$ removed and $x\in T_m$ removed from target nodes, we have one of the $c_i\in C_m$ not matched with any element of $T_m\cup\{y\}$, by definition of match-block over $T_m$, making the maximum flow between $C_m$ and $T_m\cup\{y\}\setminus\{x\}$ be $|C_m|-1$ (\cref{lemma:nonblockflow}).
There remains a full flow between $C^-_m$ and $T^-_m$, so all of the other $C$ elements have flow through them, making a path from $c_i$ not able to pass over $c_j\ne c_i$. 
This makes the full max-flow from $C$ to $T_f$ be $k-1$ by combining the match-blocked paths and non-match-blocked paths. 
Finally, all paths from $S_f$ to $y$ must cross $C$, so we have satisfied condition 2.

$\Leftarrow$: We now show that if a match-block between $C$ and $T_f\cup\{x\}$ does not exist containing $x$, then either condition 1 or 2 is violated. 
We can find sets $S_f,T_f$ satisfying condition 1 (otherwise $S_f$ has no path to $x$, and no match-block exists to $x$). The question, then, is whether any such sets also satisfy condition $2$.

Given any candidate set $S_f,T_f$ satisfying condition 1, we have the closest min-cut $C$ between $S_f$ and $T_f\cup\{x\}$. Suppose for the sake of contradiction that condition 2 is also satisfied.  
That is, all flows through $C$ to $T_f\cup\{x\}$ must pass through $x\rightarrow y$. This means that $x$ is part of the closest min-cut, ie, $x\in C$. By the theorem's conditions, all edges incoming to $c_i\in C$ were removed,
so there are no edges incoming to $x$. If $x$ has no path to any $t_i\in T_f$, then $x$ is match-blocked with itself (ie, can use $S_f=\{x\},T_f=\emptyset$), a contradiction. 
We know that the flow between $C\setminus \{x\}$ and $T_f$ must be $k-1$ (if it were $k$, we would not have condition 2 satisfied).
If there exists a flow between $C\setminus \{x\}$ and $T_f$ that does not block a path from $x$ to some $t_i \in T_f$, then condition 2 is satisfied by appending this path to the flow, constructing a new flow of size $k$, a contradiction.
Finally, in the case when all paths from $x$ to $T_f$ are blocked by all flows between $C\setminus \{x\}$ and $T_f$ we invoke \cref{lemma:blockcut} to claim that there exists a closer min-cut than $C$, a contradiction.

\end{proof}

\begin{algorithm}
    \caption{\textsc{ICID} recursively applies IC to all nodes}
    \label{alg:icid}
    \begin{algorithmic}
        \Function{ICID}{G}
        \State $\Lambda^* \leftarrow \emptyset$
        \Do
        \ForAll{$y\in G$}
        \State $(\_,\_,T_m) \leftarrow \textsc{IC}(G,y,\Lambda^*)$
        \State $\Lambda^* \leftarrow \Lambda^*\cup \{\lambda_{ty} | t\in T_m\}$
        \EndFor
        \DoWhile {at least one parameter was identified in this iteration}
        \State\Return $\Lambda^*$
        \EndFunction
    \end{algorithmic}
\end{algorithm}

\end{document}